\documentclass[twoside]{article}

\usepackage[preprint]{aistats2026}
\usepackage[round, authoryear]{natbib}

\usepackage{subcaption}
\usepackage{comment}
\usepackage{enumitem}
 
\usepackage{algorithm}
\usepackage{makecell}
\usepackage{algpseudocode}

\usepackage{amsmath,graphicx}
\usepackage{amssymb}
\usepackage{xcolor}
\usepackage{amsfonts} 
\usepackage{amsthm} %
\usepackage{caption,subcaption}
\usepackage{dsfont}
\usepackage[printonlyused]{acronym}

\usepackage{url}
\usepackage{bigints}
\usepackage{bm}
\usepackage{tcolorbox}
\tcbuselibrary{theorems}

\newtcbtheorem{boxtheorem}{Theorem}%
{colback=blue!5,colframe=gray!35!black,fonttitle=\bfseries}{th}

\newtcbtheorem{boxlemma}{Lemma}%
{colback=blue!5,colframe=gray!35!black,fonttitle=\bfseries}{th}

\newtcbtheorem{boxdefinition}{Definition}%
{colback=blue!5,colframe=gray!35!black,fonttitle=\bfseries}{th}

\usepackage{booktabs}       

\usepackage{pgf,tikz,pgfplots}
\usepackage{mathrsfs}

\definecolor{tcavred}{RGB}{210,45,45}                  
\definecolor{tcavpurple}{RGB}{105,80,180}              
\definecolor{tcavgreen}{RGB}{30,140,80}                
\definecolor{tcavorange}{RGB}{215,125,30}              
\definecolor{tcavblue}{RGB}{40,90,200}                 
\definecolor{tcavalphablue}{RGB}{20,110,170}           
\definecolor{tcavalphalight}{RGB}{120,180,215}         
\definecolor{tcavalphadeep}{RGB}{8,55,105}             
\definecolor{negfill}{RGB}{252,232,232}
\definecolor{posfill}{RGB}{226,236,250}

\pgfplotsset{compat=1.18}
\usepgfplotslibrary{groupplots}

\usetikzlibrary{arrows,shapes,calc,decorations.pathreplacing,intersections,quotes,angles,positioning,fit,petri,through, arrows.meta}
    \pgfdeclarelayer{background}
    \pgfdeclarelayer{foreground}
    \pgfsetlayers{background,main,foreground}
\usetikzlibrary{spy}
\usetikzlibrary{shapes}
\usetikzlibrary{plotmarks}

\pgfmathdeclarefunction{gauss}{2}{%
   \pgfmathparse{1/(#2*sqrt(2*pi))*exp(-((x-#1)^2)/(2*#2^2))}%
}

\pgfplotsset{compat=1.18}

\definecolor{biasblue}{RGB}{0,114,178}
\definecolor{varianceorange}{RGB}{230,159,0}
\definecolor{msegreen}{RGB}{0,135,110}

\pgfplotstableread{
k biasSqN100 varN100 mseN25 mseN100 mseN400
0.02 0.24208691 0.00006361 0.24609098 0.24215052 0.23436770
0.03 0.23818197 0.00014298 0.24419763 0.23832496 0.22680292
0.05 0.23048262 0.00039591 0.24053577 0.23087852 0.21219696
0.08 0.21922992 0.00100573 0.23536076 0.22023564 0.19164235
0.12 0.20482879 0.00222775 0.22906360 0.20705654 0.16683181
0.18 0.18463703 0.00484360 0.22091185 0.18948063 0.13513664
0.27 0.15774806 0.01013385 0.21148577 0.16788191 0.09896202
0.40 0.12614806 0.01935369 0.20309004 0.14550175 0.06549004
0.60 0.09208355 0.03372311 0.19900546 0.12580666 0.03973634
0.80 0.07080155 0.04595024 0.20119396 0.11675179 0.02844889
1.00 0.05748009 0.05572208 0.20626648 0.11320217 0.02306638
1.30 0.04579534 0.06686045 0.21550575 0.11265579 0.01943194
1.70 0.03777653 0.07747635 0.22734559 0.11525288 0.01769645
2.30 0.03223927 0.08826663 0.24162597 0.12050589 0.01716285
3.00 0.02937483 0.09659898 0.25360757 0.12597381 0.01738349
4.50 0.02705561 0.10710561 0.26926788 0.13416121 0.01826741
7.00 0.02595302 0.11564533 0.28211236 0.14159834 0.01939566
10.0 0.02555493 0.12065267 0.28961033 0.14620759 0.02020732
15.0 0.02534202 0.12474863 0.29570378 0.15009065 0.02094835
22.0 0.02525079 0.12744885 0.29969625 0.15269964 0.02147357
35.0 0.02520283 0.12965066 0.30293561 0.15485348 0.02192301
50.0 0.02518684 0.13078645 0.30460060 0.15597329 0.02216220
}\alphatcavdata

\newtheorem{theorem}{Theorem} 
\newtheorem{lemma}[theorem]{Lemma}
\theoremstyle{definition}
\newtheorem{definition}[theorem]{Definition}

\newtheorem{proposition}[theorem]{Proposition}
\newtheorem{corollary}[theorem]{Corollary}
\theoremstyle{remark}
\newtheorem{remark}[theorem]{Remark}

\usepackage[hidelinks,bookmarks=false]{hyperref}

\usepackage{amsmath,amsfonts,bm}

\def\eqref#1{equation~\ref{#1}}

\def\1{\bm{1}}

\def\vmu{{\bm{\mu}}}

\def\vw{{\bm{w}}}
\def\vx{{\bm{x}}}
\def\vy{{\bm{y}}}
\def\vz{{\bm{z}}}

\def\mA{{\bm{A}}}

\def\mC{{\bm{C}}}

\def\mX{{\bm{X}}}

\def\mSigma{{\bm{\Sigma}}}

\DeclareMathAlphabet{\mathsfit}{\encodingdefault}{\sfdefault}{m}{sl}
\SetMathAlphabet{\mathsfit}{bold}{\encodingdefault}{\sfdefault}{bx}{n}

\newcommand{\E}{\mathbb{E}}

\newcommand{\R}{\mathbb{R}}

\newcommand{\Var}{\mathrm{Var}}

\newcommand{\Cov}{\mathrm{Cov}}

\DeclareMathOperator{\Tr}{Tr}

\def\R{\mathbb{R}}
\def\P{\mathbb{P}}

\def\E{\mathbb{E}}

\def\N{\mathbb{N}}

\def\mSigma {{\mathbf{\Sigma }}}

\def\mC{{\mathbf{C}}}
\def\pat{\text{pat}}
\def\fast{\text{fast}}
\def\CAV{\text{CAV}}

\def\TCAV{{\operatorname{\operatorname{TCAV}}}}

\newcommand{\ie}{\emph{i.e.,}~}

\renewcommand{\eqref}[1]{\textup{(\ref{#1})}}

\begin{document}

%
\runningtitle{$\alpha$-TCAV: A Unified Framework for Testing with CAVs}
\runningauthor{Schnoor, Said, Tiomoko, Samek, Jung}

%

\twocolumn[

\aistatstitle{$\alpha$-TCAV: A Unified Framework for Testing with \\ Concept Activation Vectors}

\aistatsauthor{
  Ekkehard Schnoor\\[-0.1em]
  {\normalfont Department of Computer Science}\\
  {\normalfont Aalto University}\\
  {\normalfont Espoo, Finland}
  \And
  Jawher Said\\[-0.1em]
  {\normalfont Department of Artificial Intelligence}\\
  {\normalfont Fraunhofer Heinrich Hertz Institute}\\
  {\normalfont Berlin, Germany}
  \And
  Malik Tiomoko\\[-0.1em]
  {\normalfont Huawei Noah's Ark Lab}\\
  {\normalfont Huawei Technologies}\\
  {\normalfont Paris, France}
  \AND
  Wojciech Samek\\[-0.1em]
  {\normalfont Department of AI, Fraunhofer HHI}\\
  {\normalfont Department of EECS, TU Berlin}\\
  {\normalfont Berlin, Germany}
  \And
  Alexander Jung\\[-0.1em]
  {\normalfont Department of Computer Science}\\
  {\normalfont Aalto University}\\
  {\normalfont Espoo, Finland}
}

\aistatsaddress{}

]

\begin{abstract}
Concept Activation Vectors (CAV) and the Testing with CAVs (TCAV) method are fundamental tools for concept-based explainability in deep learning, yet their practical utility is limited by the variability across different samples of concept and non-concept examples.
We analyze this stochastic nature of CAVs, and identify a fundamental flaw in the standard TCAV score, its instability by using a discontinuous indicator function. To address this, we introduce $\alpha$-TCAV, a parameterized unified framework that replaces the discontinuous indicator with a smooth function, which recovers TCAV as a limiting case. Under realistic assumptions, we characterize the score distributions and derive principled parameter choices: one matches the mean of its averaged counterpart, Multi-TCAV, with no larger variance using a single CAV; the other provides a calibrated probability that the concept positively influences the model output. Our analysis and experiments provide practical guidance and challenge the established practice of splitting the sampling budget across multiple CAVs.
\end{abstract}

\vspace{-0.2cm}
\section{Introduction}
\label{sec:introduction}

\vspace{-0.2cm}

As deep neural networks have achieved impressive performance across diverse domains, understanding 
their decision-making has become critical, particularly in high-stakes applications such as medical 
diagnosis or autonomous driving; explanations of AI systems are also increasingly required by law. 
Traditional explainability methods focus on low-level feature attribution (\textit{heatmaps}) via 
techniques such as \textit{Saliency Maps} \citep{simonyan2013deep}, \textit{Grad-CAM} 
\citep{selvaraju2017grad}, and \textit{Integrated Gradients} \citep{sundararajan2017axiomatic}, 
highlighting key input regions for specific predictions. Yet such pixel-level maps diverge from 
human reasoning: they show \textit{where} the model looks, but not \textit{what} concepts it has 
detected, while humans naturally reason in terms of patterns, shapes, and colors (in a vision context) \citep{biederman1987recognition, treisman1980feature}. 
This gap motivates concept-based explanation methods, which decode internal representations in terms of 
human-understandable concepts. Among these, \textit{Concept Activation Vectors (CAVs)} 
have emerged as one of the most popular and influential techniques \citep{kim2018interpretability}, 
and are the topic of this article. Concept-based XAI has been highly active in recent years: 
\citep{ghorbani2019towards} introduced \textit{Automatic Concept Extraction} for unsupervised 
concept discovery, while \textit{Concept Relevance Propagation (CRP)} \citep{achtibat2023attribution} 
generalizes the feature-based \textit{Layer-wise Relevance Propagation (LRP)} \citep{bach2015pixel}. 
For an overview, we refer the reader to the survey \citep{poeta2023concept} and references therein.
CAVs and the Testing with CAVs (TCAV) method represent human-understandable concepts as directions in latent space. 
Consider testing whether \textit{stripes} matters for the class \textit{zebra}. TCAV answers by checking, for each zebra image, whether the sensitivity score is positive, and reporting the fraction of positive votes. The CAV is estimated from finite samples, so the sensitivity score is noisy. For a concept whose true influence is weak or neutral, the score hovers around zero, \ie the hard sign-check turns tiny fluctuations into full flips of the vote. Crucially, more data does not help, \ie the CAV itself stabilizes at the usual $O(1/N)$ rate, but the verdict stays a coin flip, because the discontinuity at zero destroys the benefit of the better estimate. Thus, CAVs can stabilize with more samples while their sign-based votes remain unstable at neutrality. 
Practitioners rank concepts, compare layers and classes, and increasingly report TCAV scores in high-stakes or regulated settings. A score whose value is not reproducible across reruns of the same analysis can silently reverse conclusions, and the instability is largest exactly for the borderline concepts where one most needs a trustworthy answer.
$\alpha$-TCAV replaces each hard vote with a soft vote: a sensitivity score far from zero still counts as $\approx 0$ or $\approx 1$, but a score near zero counts as $\approx 1/2$, the vote reflects how confident the evidence is, not just its sign.
Using linear classifiers (e.g., SVMs) and directional derivatives, TCAV quantifies concept importance for model prediction (e.g., \textit{stripes} for \textit{zebras}). This helps to address the questions:
\vspace{-0.3cm}
\begin{enumerate}[leftmargin=0.5cm, itemsep=0em]
    \item \textit{Is a concept encoded in the neural network?} Answered by the \textit{classification accuracy}
          of a CAV separating the latent (non)-concept distributions.
    \item \textit{Is a concept being used for the model's prediction?} Answered by \textit{sensitivity scores} (directional derivative in            direction of the CAV) by the \textit{TCAV} method.
\end{enumerate}
\vspace{-0.2cm}
Various CAV computation methods have been proposed \citep{dreyer2024hope}, including \textit{PatternCAV} 
\citep{pahde2025navigating, martin2019interpretable}, and \textit{FastCAV} \citep{schmalwasser2025fastcav}. 
Related extensions include \textit{Concept Bottleneck Models (CBMs)} \citep{koh2020concept}, enforcing human-defined concepts as 
intermediate representations; \citep{yuksekgonul2022post} extend this to post-hoc CBMs without retraining. 
\textit{Concept Activation Regions (CARs)} \citep{crabbe2022concept} represent concepts as regions 
rather than single vectors.  
Recently, \citep{wenkmann2025variability,schnoor2026cav} have initiated a rigorous investigation of CAVs from a probabilistic view, aiming to describe their distributions and variability. 
On a broader level, our approach parallels other efforts to theoretically ground XAI methods, such as
\citep{garreau2020explaining, garreau2020looking} for \textit{Local Interpretable Model Agnostic Explanation (LIME)}
\citep{ribeiro2016should}, or \citep{chen2025unified} for \textit{Shapley Values} \citep{shapley1953value} and their 
adaption to XAI  \citep{lundberg2017unified}. 
We also provide practical guidance on choosing a TCAV variant and deciding, revisiting established practices and
complementing previous advice \citep{nicolsonexplaining}. Our main contributions are:
\vspace{-0.2cm}
\begin{itemize}[leftmargin=0.5cm,itemsep=0em]
    \item We introduce $\alpha$-TCAV, which replaces TCAV's binary votes with smooth scores whose sharpness is controlled by $\alpha$ (Sec.~\ref{subsec:alpha_TCAV}).
    \item Under a Gaussian model, we derive the score distributions and analyze how smoothing affects bias, variance, and mean-squared error.
\item We derive two parameter choices: $\alpha^\star$ matches Multi-TCAV's mean with no larger variance using one CAV, while $\alpha^\dagger$ gives the posterior probability that the underlying sensitivity is positive.
\end{itemize}
\vspace{-0.3cm}

Sec.~\ref{sec:cavs} recalls CAVs, and Sec.~\ref{sec:tcav} how they are used in the TCAV method, before proposing $\alpha$-TCAV.
The remaining Sec.~\ref{sec:main_results} contains all main results. 
Technical details,  proofs and derivations as well as additional experiments are given in the supplementary material.

\pgfmathdeclarefunction{erf}{1}{%
  \pgfmathparse{#1 < 0 ? -1 : 1}%
  \edef\erfsign{\pgfmathresult}%
  \pgfmathparse{1/(1+0.3275911*abs(#1))}%
  \edef\erft{\pgfmathresult}%
  \pgfmathparse{%
    \erfsign * (
      1-
      (((((1.061405429*\erft-1.453152027)*\erft
      +1.421413741)*\erft-0.284496736)*\erft
      +0.254829592)*\erft)*exp(-(#1)^2)
    )
  }%
}
\begin{figure*}[ht]
\centering
\resizebox{\textwidth}{!}{%
\begin{tikzpicture}[
    >=Stealth,
    font=\sffamily,
    panel title/.style={font=\sffamily\Large\bfseries},
    panel subtitle/.style={font=\sffamily\large, text width=7cm, align=center},
    axis line/.style={->, thick, gray!70},
    cav arrow/.style={->, very thick},
    gradient arrow/.style={->, very thick, black},
    cand arrow/.style={->, gray!45, dashed, thin},
    cand arrow red/.style={->, tcavred!28, dashed, thin},
    cand arrow blue/.style={->, tcavalphablue!28, dashed, thin},
    cand arrow purple/.style={->, tcavpurple!30, dashed, thin},
    cand arrow green/.style={->, tcavgreen!40, dashed, thin},
    cand arrow orange/.style={->, tcavorange!40, dashed, thin},
    proj line/.style={dashed, gray!50, thin},
    sample dot/.style={circle, fill, inner sep=1pt},
]

\begin{scope}[shift={(0,0)}]
  \node[panel title, tcavred] at (3,5.85) {TCAV (standard)};
  \node[panel subtitle] at (3,5.25) {Hard decision on a single CAV};

  \draw[axis line] (0,0) -- (0,4);
  \draw[axis line] (0,0) -- (6,0);
  \node[anchor=south, font=\large] at (0.0,4.05) {latent space $\R^d$};
  \node[anchor=north, font=\large] at (5.0,-0.2) {$\langle \vz, \vw_\CAV\rangle$};

  \foreach \ang/\len in {33/2.8, 37/3.5, 40/2.9, 42/3.6, 48/3.0, 50/3.7, 53/2.9, 57/3.4}
      \draw[cand arrow red] (0,0) -- (\ang:\len);

  \draw[dashed, thick, black!70] (0,0) -- (32:4.5);
  \draw[gradient arrow] (0,0) -- (32:4.5);
  \node[anchor=south west, font=\large] at (32:4.6) {$\vz=\nabla h_{l,k}(f_l(\vx))$};

  \draw[cav arrow, tcavred] (0,0) -- (45:3.2);
  \node[anchor=south, font=\large, tcavred] at (45:3.3) {$\vw_\CAV$};

  \node[font=\large] at (2.4,-0.25) {$0$};
  \foreach \x in {1.4, 1.65, 1.85, 2.0, 2.1, 2.25, 2.45, 2.65, 2.9}
      \node[sample dot, tcavred] at (\x,0) {};

  \begin{scope}[shift={(0,-4.6)}]
    \node[font=\large, tcavred, anchor=south, align=center] at (3.0,2.5) {{\bfseries Hard indicator} $\mathds{1}_{\{\langle \vz, \vw_\CAV\rangle > 0\}}$};
    \fill[negfill] (0,0) rectangle (2.4,1.6);
    \fill[posfill] (2.4,0) rectangle (5,1.6);
    \draw[axis line] (0,0) -- (0,2.25);
    \draw[axis line] (0,0) -- (5.4,0);
    \node[font=\large, anchor=east] at (-0.05,2.15) {score};
    \node[font=\large] at (-0.25,0) {$0$};
    \node[font=\large] at (-0.25,1.6) {$1$};
    \node[font=\large] at (2.4,-0.3) {$0$};
    \node[anchor=north, font=\large] at (5.2,-0.15) {$\langle \vz, \vw_\CAV\rangle$};
    \draw[tcavred, ultra thick] (0,0) -- (2.4,0);
    \draw[tcavred, ultra thick, dotted] (2.4,0) -- (2.4,1.6);
    \draw[tcavred, ultra thick] (2.4,1.6) -- (5,1.6);
    \node[font=\large, tcavred, align=center] at (1.2,0.8) {negative\\vote};
    \node[font=\large, tcavblue, align=center] at (3.7,0.8) {positive\\vote};
  \end{scope}
\end{scope}

\begin{scope}[shift={(7.5,0)}]
  \node[panel title, tcavpurple] at (3,5.85) {Multi-TCAV};
  \node[panel subtitle] at (3,5.25) {Average noisy hard scores};

  \draw[axis line] (0,0) -- (0,4);
  \draw[axis line] (0,0) -- (6,0);
  \node[anchor=north, font=\large] at (5.2,-0.2) {$\langle \vz, \vw_\CAV^{(j)} \rangle$};

  \draw[dashed, thick, black!70] (0,0) -- (32:4.5);
  \draw[gradient arrow] (0,0) -- (32:4.5);
  \node[anchor=south west, font=\large] at (32:4.6) {$\vz=\nabla h_{l,k}(f_l(\vx))$};

  \draw[cand arrow orange] (0,0) -- (52:4.0);
  \draw[cand arrow orange] (0,0) -- (76:3.0);
  \draw[cav arrow, tcavorange] (0,0) -- (64:3.5);
  \node[anchor=south east, font=\large, tcavorange] at (64:3.55) {$\vw_\CAV^{(1)}$};

  \draw[cand arrow purple] (0,0) -- (26:4.1);
  \draw[cand arrow purple] (0,0) -- (46:3.2);
  \draw[cav arrow, tcavpurple] (0,0) -- (36:3.6);
  \node[anchor=south, font=\large, tcavpurple] at (36:3.7) {$\vw_\CAV^{(2)}$};

  \draw[cand arrow green] (0,0) -- (2:3.0);
  \draw[cand arrow green] (0,0) -- (20:4.1);
  \draw[cav arrow, tcavgreen] (0,0) -- (10:3.7);
  \node[anchor=south west, font=\large, tcavgreen] at (10:3.75) {$\vw_\CAV^{(s)}$};

  \node at (3.0,1.5) {$\cdots$};

  \node[font=\large] at (2.0,-0.25) {$0$};
  \foreach \x in {0.2, 0.7, 1.2} \node[sample dot, tcavorange] at (\x,0) {};
  \foreach \x in {2.0, 2.6, 3.2} \node[sample dot, tcavpurple] at (\x,0) {};
  \foreach \x in {3.7, 4.3, 4.9} \node[sample dot, tcavgreen]  at (\x,0) {};

  \begin{scope}[shift={(0,-4.6)}]
    \node[font=\large, tcavpurple, anchor=south, align=center] at (3.0,2.5) {{\bfseries Average} $\tfrac{1}{s}\sum_{j=1}^{s} \mathds{1}_{\{\langle \vz, \vw_\CAV^{(j)}\rangle > 0 \}}$};
    \fill[negfill] (0,0) rectangle (2.4,1.6);
    \fill[posfill] (2.4,0) rectangle (5,1.6);
    \draw[axis line] (0,0) -- (0,2.25);
    \draw[axis line] (0,0) -- (5.4,0);
    \node[font=\large] at (-0.25,0) {$0$};
    \node[font=\large] at (-0.25,1.6) {$1$};
    \node[font=\large] at (2.4,-0.3) {$0$};
    \node[anchor=north, font=\large] at (5.2,-0.15) {$\langle \vz, \overline{\vw}_\CAV \rangle$};
    \draw[tcavpurple, very thick]
      (0,0) -- (1.0,0) -- (1.0,0.32) -- (1.8,0.32) -- (1.8,0.64) --
      (2.6,0.64) -- (2.6,0.96) -- (3.4,0.96) -- (3.4,1.28) -- (4.2,1.28) --
      (4.2,1.6) -- (5.0,1.6);
    \node[font=\normalsize, tcavpurple] at (3.0,2.05)
      {$\overline{\vw}_\CAV = \tfrac{1}{s}\sum_{j=1}^{s} \vw_\CAV^{(j)}$};
  \end{scope}
\end{scope}

\begin{scope}[shift={(15,0)}]
  \node[panel title, tcavalphablue] at (3,5.85) {$\alpha$-TCAV (our framework)};
  \node[panel subtitle] at (3,5.25) {Smooth decision on a single CAV};

  \draw[axis line] (0,0) -- (0,4);
  \draw[axis line] (0,0) -- (6,0);
  \node[anchor=north, font=\large] at (5.0,-0.2) {$\langle \vz, \vw_\CAV\rangle$};

  \foreach \ang/\len in {33/2.8, 37/3.5, 40/2.9, 42/3.6, 48/3.0, 50/3.7, 53/2.9, 57/3.4}
      \draw[cand arrow blue] (0,0) -- (\ang:\len);

  \draw[dashed, thick, black!70] (0,0) -- (32:4.5);
  \draw[gradient arrow] (0,0) -- (32:4.5);
  \node[anchor=south west, font=\large] at (32:4.6) {$\vz=\nabla h_{l,k}(f_l(\vx))$};

  \draw[cav arrow, tcavalphablue] (0,0) -- (45:3.2);
  \node[anchor=south, font=\large, tcavalphablue] at (45:3.3) {$\vw_\CAV$};

  \node[font=\large] at (2.4,-0.25) {$0$};
  \foreach \x in {1.4, 1.65, 1.85, 2.0, 2.1, 2.25, 2.45, 2.65, 2.9}
      \node[sample dot, tcavalphablue] at (\x,0) {};

  \begin{scope}[shift={(0,-4.6)}]
    \node[font=\large, tcavalphablue, anchor=south, align=center] at (3.0,2.5) {{\bfseries Scaled probits} $\Phi(\alpha \langle \vz, \vw_\CAV\rangle)$};
    \fill[negfill] (0,0) rectangle (2.4,1.6);
    \fill[posfill] (2.4,0) rectangle (5,1.6);
    \draw[axis line] (0,0) -- (0,2.25);
    \draw[axis line] (0,0) -- (5.4,0);
    \node[font=\large] at (-0.25,0) {$0$};
    \node[font=\large] at (-0.25,1.6) {$1$};
    \node[font=\large] at (2.4,-0.3) {$0$};
    \node[anchor=north, font=\large] at (5.2,-0.15) {$\langle \vz, \vw_\CAV\rangle$};
    \draw[tcavalphalight, ultra thick, domain=0:5, samples=80, smooth]
      plot(\x, {0.8*(1+erf(1.0*(\x-2.4)/sqrt(2)))});
    \draw[tcavalphablue, ultra thick, domain=0:5, samples=80, smooth]
      plot(\x, {0.8*(1+erf(2.0*(\x-2.4)/sqrt(2)))});
    \draw[tcavalphadeep, ultra thick, domain=0:5, samples=80, smooth]
      plot(\x, {0.8*(1+erf(4.0*(\x-2.4)/sqrt(2)))});
    \node[font=\normalsize, tcavalphadeep,  anchor=west] at (0.25,1.45) {$\alpha=2$};
    \node[font=\normalsize, tcavalphablue,  anchor=west] at (0.25,1.18) {$\alpha=1$};
    \node[font=\normalsize, tcavalphalight, anchor=west] at (0.25,0.91) {$\alpha=\tfrac{1}{2}$};
  \end{scope}
\end{scope}

\end{tikzpicture}%
}
\caption{%
For an input $\vx$, the gradient $\vz = \nabla h_{l,k}(f_l(\vx))$ of logit $k$ at layer $l$ is projected onto a (random) CAV $\vw_\CAV \in \R^{d}$ ($\R^{d}$ represented as the $y$-axis), yielding the sensitivity scores $\langle \vz, \vw_\CAV\rangle$ (shown on the $x$-axis). Next, different TCAV scores are computed, where \textcolor{tcavred}{\textbf{TCAV}} (left) uses a hard, overconfident ($0$ or $1$) indicator on a single low-variance CAV, potentially leading to large (or even non-vanishing) variance. \textcolor{tcavpurple}{\textbf{Multi-TCAV}} (middle) averages several such indicators using several (more noisy) CAVs, which is computationally more expensive and yields underconfident scores.
Finally, \textcolor{tcavalphablue}{\textbf{$\alpha$-TCAV (ours)}} (right) replaces the discontinuous indicator function with a smooth probit family  that also accounts for the magnitude of $\langle \vz, \vw_\CAV\rangle$, using the full budget to learn a single low-variance CAV; its temperature $\alpha$ controls the steepness and can be tuned for variance reduction and calibration.
Note that we consider two notions of \textit{variance}: firstly, for a (random) CAV $\vw_\CAV$ in the sense of Definition \ref{def:variance_of_CAV}; secondly, for the sensitivity scores $\langle \vz, \vw_\CAV\rangle$ in the usual sense of a real-valued random variable.
}
\label{fig:tcav_comparison}
\end{figure*}

\vspace{-0.2cm}
\section{Concept Activation Vectors}
\label{sec:cavs}

\vspace{-0.2cm}

\paragraph{Neural Network Architecture and CAVs.} 
Consider a (pretrained) $L$-layer neural network for a $K$-class classification task, represented by the composition $h_l \circ f_l$. Here, $f_l: \R^{d_0} \to \R^{d_l}$ maps inputs to layer $l$ activations, and $h_l: \R^{d_l} \to \R^{d_L}$ maps these to the output (where $d_L = K$). We denote the mapping to the $k$-th class logit as $h_{l,k}: \R^{d_l} \to \R$. 
Concept Activation Vectors (CAVs) are derived by training a linear classifier to separate activations of a concept $C$ from non-concept examples. The resulting CAV, $\vw_\CAV \in \R^{d_l}$, is the vector orthogonal to the separating hyperplane pointing toward the concept region (we also write $ \vw_\CAV$ to denote generic CAVs). We consider a general binary setup with classes $\mathcal{C}_\ell$ for $\ell \in \{1, 2\}$, where $\mathcal{C}_1$ (label $-1$) represents non-concept data and $\mathcal{C}_2$ (label $1$) represents the concept class; labels are thus written as $(-1)^\ell$. We assume both classes consist of i.i.d. samples from unknown distributions, where $N$ denotes the effective number of random samples (see Rem. \ref{remark:sample_size_asymptotics}). The dataset for class $k$ is $\mathcal{X}_k$, sampled from a generic distribution $\mathcal{D}_k$;
we mainly require the gradients (directional derivatives) appearing in the sensitivity scores to exist almost surely.
We interpret CAVs as random vectors with distribution $\mathcal{D}_\CAV$.
The indicator function of a set $A$ is written as $\mathds{1}_A$.
We adopt a compact generic notation for the binary classification problem using a data matrix $\mX = [\vx_1, \dots, \vx_n] \in \R^{d \times n}$ with labels $\vy \in \{-1, 1\}^n$. Let $\mX = [\mX^{(1)}, \mX^{(2)}]$, where $\mX^{(\ell)} \in \R^{d \times n_\ell}$ gathers $n_\ell$ vectors for class $\mathcal{C}_\ell$. The total size is $n = n_1 + n_2$, with class probabilities $c_\ell = n_\ell/n$. Each class is characterized by a mean $\vmu_\ell \in \R^d$, covariance $\mSigma_\ell \in \R^{d \times d}$, and generalized covariance $\mC_\ell = \mSigma_\ell + \vmu_\ell \vmu_\ell^\top$. We write $\vx \in \mathcal{C}_\ell$ to indicate $\vx$ follows the distribution of class $\mathcal{C}_\ell$.
We denote by $\Phi$ the cumulative distribution function of the standard normal distribution and by $\Phi_2(\cdot,\cdot;\rho)$ the cumulative distribution function of the bivariate standard normal distribution with correlation $\rho\in(-1,1)$; for convenience, we recall them in \eqref{eq:cdf_gaussian} and \eqref{eq:cdf_bivariate_gaussian}, respectively.

\begin{remark}[Sample Sizes and Asymptotics]
\label{remark:sample_size_asymptotics}
Balanced sampling, $n_1=n_2$, is natural as non-concept examples are typically cheap to obtain. 
Alternatively, one may fix the number $n_2$ of concept examples and let $n_1\to\infty$, as studied for logistic-regression CAVs in \citep{owen2007infinitely,wenkmann2025variability}. 
Unlike that fixed-concept setting, we treat both concept and non-concept samples as random. 
We use $N$ for the effective random sample size $N=n_1+n_2$ in our general setting and $N=n_1$ when only the non-concept sample is random. 
Extreme imbalance may, however, make classification accuracy misleading, since the majority-class
classifier then has asymptotically zero error \citep{he2009learning,buda2018systematic}.
\end{remark}
 
\paragraph{Methods for CAV Computations.} Let us recall several methods to compute CAVs, formulated as general binary classification problems, 
where the data matrix  $\mX = [\vx_1, \dots, \vx_n] \in \R^{d \times n}$ contains the collection of the latent representations of the 
(non-)concept input examples in the CAV context. Typical choices are
\begin{align}
    \vw_{\pat} 
& =   \frac{1}{n_2} \sum_{i=1}^{n_2} \vx_i^{(2)} - \frac{1}{n_1} \sum_{i=1}^{n_1} \vx_i^{(1)}
=   \hat{\vmu}_2 - \hat{\vmu}_1, \label{eq:w_pat_empirical} \\
 \vw_{\fast} 
   & = \frac{1}{n_2} \sum_{i=1}^{n_2} \left( \vx_i^{(2)} - \hat{\vmu}_{1,2} \right)  
    = \hat{\vmu}_2 - \hat{\vmu}_{1,2}, \label{eq:fastCAV_definition} 
\end{align}
 
the \textit{PatternCAV} \citep{pahde2025navigating, martin2019interpretable}  in \eqref{eq:w_pat_empirical}, \ie the difference of the empirical means  $\hat{\vmu}_2$ and  $\hat{\vmu}_1$, showing in the direction of the concept data; see also \citep{haufe2014interpretation} for the related construction of a \textit{filter} in the neuroscience community. Finally, \eqref{eq:fastCAV_definition} is the \textit{FastCAV} 
\citep{schmalwasser2025fastcav}; the joint empirical mean $\hat{\vmu}_{1,2}$ is 
\begin{equation}
    \hat{\vmu}_{1,2} = \frac{1}{n_1 + n_2} \sum_{\ell=1}^{2} \sum_{i=1}^{n_\ell} \vx_i^{(\ell)} \in \R^d.
    \label{eq:fastCAV_joint_empirical_mean_of_classes}
\end{equation}

\paragraph{Variability of CAVs.}
A measure of the stochastic nature of a CAV is its total variance, the sum of the variances of its entries, \ie its covariance trace, 
as proposed in \citep[Def. 1]{wenkmann2025variability}.
\begin{definition} 
\label{def:variance_of_CAV}
We define the variance of a CAV $\vw_{\CAV} \in \R^d$ (assuming that its covariance exists) as
\begin{equation}
    \Var(\vw_{\CAV}) := \Tr\left(\Cov( \vw_{\CAV})\right).
\end{equation}
\end{definition}
This definition extends to arbitrary random vectors; this definition only highlights its use in the context of CAVs.
By \citep{wenkmann2025variability, schnoor2026cav}, usual CAV variants show a standard statistical behavior in the sense that
\begin{equation}
    \Var(\vw_{\CAV}) = \mathcal{O}(1/N),
    \label{eq:CAV_standard_variance_decay1/N}
\end{equation}
\ie their variance decreases at a rate\footnote{For a nonnegative sequence $(a_N)_{N \in \N}$, the notation
$a_N=\mathcal{O}(N^{-1})$ as $N\to\infty$ means that there exist
constants $C>0$ and $N_0$ such that $a_N\leq C/N$ for all $N\geq N_0$.}
of $1/N$ when the (effective) sample size $N$ increases.
In particular, it has been shown in \citep[Thm. 2 \& Cor. 4]{wenkmann2025variability} that $\Var(\vw_{\CAV}) = \mathcal{O}(n_1^{-1})$ both for the \textit{PatternCAV} (there, referred to as the \textit{Difference of Means} \citep{martin2019interpretable}), and the logistic-regression CAV, under mild assumptions, where $N = n_1$ (deterministic concept data); similar results are immediately obtained from \citep[Prop. 1.]{schnoor2026cav}.
This variance decay seems to hold quite universally for different realistic approaches to compute CAVs, \ie CAVs stabilize with increasing sample size. While this behavior is an expected behavior, the same variance decay does \textit{not} necessarily translate to the TCAV scores computed from CAVs. We address this in the next section, and then propose an antidote; more details can be found in Sec. \ref{app:sec_cav_distributions}.

\section{Testing with CAVs}
\label{sec:tcav}


After recalling basics on CAVs in the previous section, Sec.~\ref{subsec:sensitivity_TCAV_and_multi_run_TCAV} recalls TCAV and Multi-TCAV, and discusses their (probabilistic) interpretation. Sec.~\ref{subsec:alpha_TCAV} introduces
$\alpha$-TCAV. Then, we are prepared for the main results and practical recommendations in Sec.~\ref{sec:main_results}.

\subsection{Sensitivity Scores and TCAV Variants}
\label{subsec:sensitivity_TCAV_and_multi_run_TCAV}

To quantify the contribution of a concept $C$ to the prediction of class $k$ for a single input $\vx \in \R^{d_0}$, \citep{kim2018interpretability} proposed the following sensitivity score $S_{C,k,l}(\vx)$, based on a CAV $\vw_\CAV$ at layer $l$, defined as
\begin{align}
        S_{C,k,l} \left(\vx\right)
    = & \lim_{{h \to 0}} \frac{h_{l,k} \left( f_l (\vx) + h \cdot \vw_\CAV \right) - h_{l,k} \left( f_l (\vx)\right)}{h} \nonumber \\
    = & \left\langle \nabla h_{l,k} \left( f_l (\vx)\right), \vw_\CAV \right\rangle. 
    \label{eq:TCAV_sensitivity_inner_product}
\end{align}
The equivalence of the directional derivative of $h_{l,k}$ along $\vw_\CAV$ to the inner product holds under appropriate smoothness 
assumptions, which we assume to hold thereafter. For the remainder of this work, we primarily utilize the inner product formulation for its computational convenience, or simply use the generic expression $S_{C,k,l} (\vx)$ itself.
To emphasize the dependence on the (random) CAV, we may also write $S_{C,k,l} (\vx, \vw_\CAV)$ or $S_{C,k,l} (\vx, \vw_\CAV)$. 
While $S_{C,k,l} (\vx)$ considers only a single input $\vx$, \citep{kim2018interpretability} also introduced the $\operatorname{TCAV_Q}_{C,k,l}$ score to measure the overall importance of a concept $C$ to class $\mathcal{X}_k$, dubbed \textit{Testing with Concept Activation Vectors (TCAV)}. It quantifies the fraction of inputs $\vx \in \mathcal{X}_k$ for which the sensitivity score is positive (for a given CAV realization to obtain $S_{C,k,l}$),
\begin{align}
       \operatorname{TCAV_Q}_{C,k,l}  
=   &   \frac{\left| \{ \vx \in \mathcal{X}_k \, : \, S_{C,k,l} (\vx) > 0 \}\right|}{|\mathcal{X}_k|} \label{eq:TCAV_def_original_kim}  \\
=   &   \frac{1}{|\mathcal{X}_k|} \sum_{\vx \in \mathcal{X}_k} \mathds{1}_{\{ S_{C,k,l}(\vx) > 0 \}}, \label{eq:TCAV_def_original_indicator_fcts}  
\end{align}
so that $\operatorname{TCAV_Q}_{C,k,l} \in [0,1]$. 
We comment on the practical interpretation of this quantity in Remark \ref{rem:TCAV_practical_interpretation} below.
The original formulation \eqref{eq:TCAV_def_original_kim} was provided in \citep{kim2018interpretability}. However, the equivalent reformulation \eqref{eq:TCAV_def_original_indicator_fcts} as an average of indicator functions turns out to be insightful: despite being a simple rephrasing, it highlights the mean-like character of this quantity. This will inspire us to formally define this approach in the language of probability theory as an expectation
of a random variable containing the indicator function; we will then take the opportunity to propose a refined version employing a smooth variant below.
Notably, both $S_{C,k,l} (\vx)$ and $\operatorname{TCAV_Q}_{C,k,l}$ depend on the specific choice of the CAV $\vw_\CAV$, which in turn -
even for fixed \textit{(non-)concept} examples and a fixed layer $l$, still depends on a specific method and number of sample used to compute it.
More formally, \eqref{eq:TCAV_def_original_kim} and \eqref{eq:TCAV_def_original_indicator_fcts} could be interpreted as an estimator (using a single CAV realization and the available samples from $\mathcal{X}_k$ from $\mathcal{D}_k$)
\begin{align}
     \overline{\TCAV} 
:= &  \E_{\begin{subarray}{l} \vx \sim \mathcal{D}_k \\ \vw_\CAV \sim \mathcal{D}_\CAV \end{subarray}}
    \left[  \mathds{1}_{\left( S_{C,k,l}(\vx, \vw_\CAV) \right) > 0}   \right] \\ 
=  & \P_{\begin{subarray}{l} \vx \sim \mathcal{D}_k \\ \vw_\CAV \sim \mathcal{D}_\CAV \end{subarray}}
    \left(   S_{C,k,l}(\vx, \vw_\CAV)  > 0  \right),
  \label{eq:true_standard_TCAV}
\end{align}
where $ \mathcal{D}_\CAV$ denotes the (induced) distribution of the CAV, a random vector in latent space.
Similar to the ``standard'' version in \eqref{eq:TCAV_def_original_indicator_fcts}, 
we may also use an alternative estimator of \eqref{eq:true_standard_TCAV} that additionally performs an averaging 
over different realization of the CAV. This exactly is the idea of Multi-TCAV,
which will be introduced below.
This formalization also highlights the sources of randomness of the sensitivity and TCAV scores: While the model itself (specifically the functions $h_{l,k}$ and $f_l$) is assumed to be deterministic, the distribution of $S_{C,k,l}(\vx)$ originates from, firstly, $\vx$ being a random vector, drawn from the input data distribution, and  secondly, the CAV $\vw_\CAV$ being a random vector, whose distribution is induced by both the \textit{concept} and the \textit{non-concept} distribution at layer $l$.

\begin{remark}[Interpretation and Limitations of TCAV]
\label{rem:TCAV_practical_interpretation}
$\operatorname{TCAV}$ measures the prevalence of positive directional derivatives, and values near $1$, $0.5$, and $0$ indicate the concept's predominantly positive, neutral, and negative influence, respectively. However, the sign-based definition \eqref{eq:TCAV_def_original_kim} ignores magnitudes and is discontinuous at zero. In particular, negative and zero sensitivities both yield $\operatorname{TCAV_Q}_{C,k,l}=0$. This discontinuity in the origin can lead to non-decaying variance even as $N \to \infty$ \citep[Sec.~4.4]{wenkmann2025variability}.
\end{remark}

\paragraph{Multi-TCAV.}
To decrease the variance by averaging, Multi-TCAV was already proposed in the seminal paper \citep{kim2018interpretability} and remains the state of the art until now. It partitions the total number $N = s \cdot n$ of (non-)concept examples into $s$ disjoint subsets of size $n$. 
For each subset $j \in \{1, \dots, s\}$, a separate CAV is trained, providing an individual TCAV score $T_j$, 
\begin{align}
 T_j  = & \operatorname{TCAV_Q}_{C,k,l} \left( \vw_{\CAV}^{(j)} \right) \\ 
      = & \frac{\left| \{ \vx \in \mathcal{X}_k \, : \, S_{C,k,l} \left(\vx, \vw_{\CAV}^{(j)}\right) > 0 \}\right|}{|\mathcal{X}_k|}.    
      \label{eq:T_J}
\end{align}
The final multi-run TCAV score, $T_{\operatorname{multi}}$, is defined as the empirical average of these individual scores, \ie 
\begin{equation*}
     T_{\operatorname{multi}} 
=   \sum_{j=1}^s  \frac{T_j}{s}  \\
=     \frac{1}{s |\mathcal{X}_k|} \sum_{j=1}^s \sum_{\vx \in \mathcal{X}_k} \mathds{1}_{\{  S_{C,k,l} (\vx, \vw^{(j)}) > 0 \}}.    
\end{equation*}
Typically  $\vw_{\CAV}^{(1)}, \dots \vw_{\CAV}^{(s)}$ have the same mean as a single global CAV computed with the identical method using all samples, but this is generally \textit{not} true for the covariance (see e.g. Proposition \ref{prop:distribution_fast_pattern_CAV}).
Naturally, recalling Def. \ref{def:variance_of_CAV}, we typically have a stronger variance decay $\Var(\vw_{\CAV})= \Tr(\Cov( \vw_{\CAV})) = \mathcal{O}(1/N)$ for the single CAV, in contrast to the larger variance $\Var(\vw_{\CAV}^{(j)})= \Tr(\Cov( \vw_{\CAV}^{(j)})) = \mathcal{O}(1/n)$ in the multi-CAV approach for $j=1, \dots, s$.
Regarding the variance of Multi-TCAV, \citep[Conj. 1]{wenkmann2025variability} states that
$\Var(T_{\operatorname{multi}}) = \mathcal{O}(1/s)$, which indeed is a straightforward consequence of 
Lemma \ref{lem:variance_empirical_average} for independent $T_j$ and assuming $\Var(T_j) = \mathcal{O}(1)$ for all $j$. While this scaling with respect to $s$ seems overall correct, it misses a mild dependency on $n$. 
Given a fixed sampling budget $N = s \cdot n$, increasing $s$ necessarily decreases $n$, thus typically increases the above-mentioned $\Var(\vw_{\CAV}^{(j)})$ of the individual CAVs, and thus variances of $T_j$.
Thus, increasing $s$ also increases the variances of $T_j$ over which we average in $ T_{\operatorname{multi}}$. 
This subtle tradeoff suggests that Multi-TCAV may not be optimal.

\subsection{$\alpha$-TCAV}
\label{subsec:alpha_TCAV}

\begin{figure} 
    \centering
    \begin{tikzpicture}
        \begin{axis}[
            axis lines = middle,
            xmin=-5.5, xmax=5.5,
            ymin=-0.2, ymax=1.3,
            xtick={0},
            ytick={0.5, 1},
            every axis x label/.style={at={(ticklabel* cs:1.05)}, anchor=west},
            xlabel={$x$},
            ylabel={}, 
            axis line style={very thick},
            width=9cm, height=5.8cm,
            legend style={at={(1.0, 0.16)}, anchor=south east, draw=none, fill=none, row sep=2pt, align=left}
        ]
        \addplot [red, very thick, dashed, domain=-5.5:5.5, samples=100]
            {0.5*(1 + erf(x/sqrt(2)))};
        \addlegendentry{$\Phi(x)$}
        
        \addplot [green!60!black, ultra thick, dashed, domain=-5.5:5.5, samples=150]
            {0.5*(1 + erf(3*x/sqrt(2)))};
        \addlegendentry{$\Phi (\alpha x)$ \\ $\alpha > 1$}
        
        \addplot [blue, very thick, domain=-5.5:-0.02, forget plot] {0};
        \addplot [blue, very thick, domain=0.02:5.5] {1};
        \addlegendentry{$\Phi_\infty(x)$}
        
        \addplot [orange, very thick, dashdotted, domain=-5.5:-0.02, forget plot] {0};
        \addplot [orange, very thick, dashdotted, domain=0.02:5.5] {1};
        \addlegendentry{$\mathds{1}_{\{x > 0\}}$}
        
        
        \draw[blue, thick, fill=white] (axis cs:0,0) circle (2.5pt);
        \draw[blue, thick, fill=white] (axis cs:0,1) circle (2.5pt);
        \draw[blue, thick, fill=blue] (axis cs:0,0.5) circle (2pt); 
        
        \draw[orange, thick, fill=orange] (axis cs:0,0) circle (1.2pt);
        
        \node[orange, thick] at (axis cs:0,1) {\scalebox{1.5}{$\times$}};
        
        \node[below left] at (axis cs:0,0) {0};
        
        \end{axis}
    \end{tikzpicture}
    \caption{Illustration of the indicator function $\mathds{1}_{\{x > 0\}}$ and approximations by the probit function $\Phi$, the scaled probit function 
    as well as the pointwise limit $\Phi_\infty(x) := \lim_{\alpha \to \infty} \Phi(\alpha x)$, the \textit{Heaviside function}.}
    \label{fig:probit_indicator_fct}
\end{figure}

As an alternative to the standard TCAV procedure, and as a generalized framework, we introduce $\alpha$-TCAV.
Its fundamental idea is to replace the indicator function in \eqref{eq:true_standard_TCAV} by the \textit{scaled Gaussian CDF}
$\Phi(\alpha x)$ for $\alpha > 0$,
where $\Phi$ is the standard Gaussian CDF defined in \eqref{eq:cdf_gaussian}.
Thus, as an alternative definition to \eqref{eq:true_standard_TCAV} we suggest
\begin{equation}
    \overline{\TCAV}_\alpha
=  \E_{\begin{subarray}{l} \vx \sim \mathcal{D}_k \\ \vw_\CAV \sim \mathcal{D}_\CAV \end{subarray}}
    \left[ \Phi \left( \alpha S_{C,k,l}(\vx, \vw_\CAV) \right) \right].   
    \label{eq:def_alpha_TCAV}
\end{equation}

While the proposed $\alpha$-$\TCAV$ may be regarded as a smooth approximation of \eqref{eq:TCAV_def_original_indicator_fcts}, we think it is a refined definition in first place. It is better at capturing the neutral behavior around the origin $S_{C,k,l} (\vx) = 0$, by discounting sensitivity scores close to zero, thus not only taking the sign of $S_{C,k,l} (\vx)$ into account, but also its magnitude; compare again Rem. \ref{rem:TCAV_practical_interpretation}. It further reduces the variance \textit{without} the need for computing several CAVs as in Multi-TCAV;
see Sections \ref{app:normalization} and \ref{app:class_level_rb} in the supplementary material.
The definition of $\alpha$-$\TCAV$ as given in \eqref{eq:def_alpha_TCAV} is generic as it considers
an expectation with respect to the underlying data distribution $\vx \sim \mathcal{D}_k$ as well as the distribution of the CAV $\vw_\CAV$.
It is straightforward to consider specific variants of this, like assuming a fixed (deterministic) set of concept examples, and sampling only
from the non-concept distribution \citep{wenkmann2025variability}.
We may also consider a \textit{conditional} version (conditionally on the random CAV  $\vw_\CAV$), by taking the expectation with respect to $\vx \in \mathcal{X}_k$, making the expectation in \eqref{eq:def_alpha_TCAV} a random variable (or vice versa).

\begin{table*}[t]
\centering
\renewcommand{\arraystretch}{2.8}  
\small  
\begin{tabular}{cclll}
\toprule
\makecell[c]{\textbf{Method} \\[2pt] Distribution}
& \textbf{Shape}
& \textbf{Random Variable}
& \textbf{Mean}
& \textbf{Variance} \\ 
\midrule
\makecell[c]{\textbf{TCAV} \\[2pt] Bernoulli}
&
\makecell[c]{two-point/\\ bimodal}
& 
$\mathds{1}_{\left\{\tfrac{1}{N}\sum_{j=1}^{N} X_j > 0\right\}}$
& 
$\Phi\left( \frac{\sqrt{N} \mu}{\sigma} \right)$
& 
$\Phi\left( \frac{\sqrt{N}\mu}{\sigma} \right)
\left(1-\Phi\left( \frac{\sqrt{N}\mu}{\sigma} \right)\right)$  \\
\midrule
\makecell[c]{\textbf{Multi-TCAV} \\[2pt] (norm.) Binomial}
&
\makecell[c]{unimodal}
& 
$\frac{1}{s}\sum\limits_{i=1}^{s}
\mathds{1}_{\left\{\tfrac{1}{n}\sum_{j=1}^{n}X_j^{(i)}>0\right\}}$
& 
$\Phi\left( \frac{\sqrt{n}\mu}{\sigma} \right)$
& 
$\frac{
\Phi\left( \frac{\sqrt{n}\mu}{\sigma} \right)
\left(1-\Phi\left( \frac{\sqrt{n}\mu}{\sigma} \right)\right)
}{s}$ \\
\midrule
\makecell[c]{\textbf{$\alpha$-\textbf{TCAV}} \\[2pt] Probit-Normal}
&
\makecell[c]{%
$\begin{array}{r@{\;:\;}l}
\text{unimodal} & \alpha<\alpha^\dagger \\
\text{monotone} & \alpha=\alpha^\dagger \\
\text{bimodal}  & \alpha>\alpha^\dagger
\end{array}$
}
& 
$\Phi\left(\alpha\frac{1}{N}\sum_{j=1}^{N}X_j\right)$
& 
$\Phi\left(\mu\rho_\alpha\right)$
& 
$\Phi_2\left(
\mu\rho_\alpha,\,
\mu\rho_\alpha;\,
\rho_\alpha^2\sigma^2/N
\right)
-\Phi^2\left(\mu\rho_\alpha\right)$ \\
\bottomrule
\end{tabular}
\vspace{0.6em}
\caption{Characterization of the three TCAV variants under the model \eqref{eq:gaussian_model_main_part} for the sensitivity score $\langle \vz,\vw_\CAV\rangle$, each corresponding to a classical probability distribution. Here $\Phi$ and $\Phi_2$ denote the univariate and bivariate standard normal cdfs defined in \eqref{eq:cdf_gaussian} and \eqref{eq:cdf_bivariate_gaussian}, respectively, $N=sn$ is the total sample budget split into $s$ batches of size $n$, $\rho_\alpha:=\alpha/\sqrt{1+\alpha^2\sigma^2/N}$, and $\alpha^\dagger:=\sqrt{N}/\sigma$ is the calibrated temperature; see \eqref{eq:alpha_tcav_moment_parameters} and \eqref{eq:alpha_dagger_main_part}. At the transition $\alpha=\alpha^\dagger = \sqrt{N}/\sigma$ in \eqref{eq:alpha_dagger_main_part}, the Probit-Normal density --- which we recall in  Sec.~\ref{sec:probability} --- is uniform when $\mu=0$ and monotone when $\mu \neq 0$. The normalized Binomial distribution is unimodal (except for ties). }
\label{table:TCAV_distributions_main_part}
\end{table*}

\vspace{-0.2cm}
\section{Main Results}
\label{sec:main_results}

\vspace{-0.2cm}

A key theme of this paper is the stochastic nature of CAVs, which carries over to the various TCAV scores discussed above. 
Their variances all satisfy the trivial yet instructive upper bound  
\begin{equation}
    \Var(\TCAV) \leq \frac{1}{4},
    \label{eq:variance_TCAV_trivial_upper_bound}
\end{equation}
which holds for any arbitrary random variable $\TCAV$ taking values in $[0,1]$, by the classical \textit{Popoviciu's Inequality} 
\citep{popoviciu1935equations}; see Thm.~\ref{thm:popoviciu}.
The upper bound is attained by a symmetric Bernoulli distribution. Although elementary, 
this observation highlights a key shortcoming of the classical TCAV in 
\eqref{eq:TCAV_def_original_indicator_fcts}: due to the indicator function, it follows a 
Bernoulli distribution, which is prone to high variance when (close to) symmetric, achieving low variance 
only by concentrating most mass on $0$ or $1$ (see Fig.~\ref{fig:variance_illustration}).
Next, we want to study the statistical properties of the TCAV variants in more detail, and derive practical recommendations.

\vspace{-0.2cm}

\paragraph{Gaussian Model.}
For a more detailed investigation, we assume the following mathematical model for the (conditional, for fixed input $\vx$) sensitivity scores $S_{C,k,l}(\vx) = \langle \nabla h_{l,k}(f_l(\vx)), \vw_\CAV \rangle = \langle \vz, \vw_\CAV \rangle$, 
assuming $\vz \neq 0$. We will consider the distribution
\begin{equation}
\langle \vz, \vw_\CAV \rangle \, \sim \,  S_N := \frac{1}{N}\sum_{j=1}^N X_j \sim \mathcal{N}(\mu, \sigma^2/N),
\label{eq:gaussian_model_main_part}
\end{equation}
with i.i.d.\ $X_j \sim \mathcal{N}(\mu, \sigma^2)$.  
(Intuitively, $\mu = 0$ denotes a neutral concept, whereas $\mu > 0$ models a positive contribution of the concept; vice versa for $\mu < 0$. Due to a noisy CAV $\vw_\CAV$, the computed score fluctuates around $\mu$, but stabilizes as $N$, the number of samples being used to obtain the CAV, grows.) This appropriately models the variance decay (Def.~\ref{def:variance_of_CAV}) of the CAV \eqref{eq:CAV_standard_variance_decay1/N}
correctly, as $\Var(\langle \vz, \vw_\CAV \rangle) = \vz^\top \Cov(\vw_\CAV)\, \vz \leq \|\vz\|^2 \Tr(\Cov(\vw_\CAV)) = \mathcal{O}(N^{-1})$. 
Further, by standard CLT-based arguments, $\langle \vz, \vw_\CAV \rangle$ is asymptotically normal for fixed $\vz$ by
\citep[Thm.~1 \& Cor.~1]{wenkmann2025variability}; see also Sec. \ref{app:sec_cav_distributions}.
This is generally \textit{not} true for the joint distribution, when also the input $\vx$ (and therefore $\vz$) is random, which we address below.
Splitting into $s$ batches of size $n$ (so $N = sn$) analogously models Multi-TCAV via batch-averages $\tfrac{1}{n}\sum_{j=1}^n X_j^{(r)} \sim \mathcal{N}(\mu, \sigma^2/n)$, $r = 1,\dots, s$. We obtain a mathematically elegant link between each TCAV variant and a corresponding 
classical probability distribution: \textit{Bernoulli} (TCAV), \textit{Binomial} on $\{0, \tfrac{1}{s}, \dots, 1\}$ (Multi-TCAV), and \textit{Probit-Normal} ($\alpha$-TCAV). We provide their exact distributions in Table \ref{table:TCAV_distributions_main_part}; detailed derivations are given in the appendix.

\begin{figure*}[t]
\centering
\def\alphaFigureBody{%
\begin{tikzpicture}
\begin{groupplot}[
    group style={
        group name=alphafig,
        group size=5 by 1,
        horizontal sep=0.55cm
    },
    scale only axis,
    width=\alphaPanelWidth,
    height=4.75cm,
    xmin=0,
    xmax=1,
    enlarge x limits=0.025,
    xtick={0.5,1},
    xticklabels={$0.5$,$1$},
    axis lines=left,
    axis line style={thick},
    tick label style={font=\small},
    label style={font=\small},
    title style={font=\small,align=center,yshift=1pt},
    grid=major,
    grid style={gray!20},
    clip=false
]

\nextgroupplot[
    title={$\alpha$-TCAV\\$\alpha\to0$\\[-1pt]
    {\scriptsize\itshape\color{black!60}indifferent}},
    ymin=0,
    ymax=2.2,
    ytick={1,2},
    ylabel={density / mass}
]
\addplot+[
    orange!85!black,
    ultra thick,
    ycomb,
    mark=*,
    mark size=2pt
] coordinates {
    (0.5,1)
};

\nextgroupplot[
    title={Multi-TCAV\\
    $\alpha=\alpha^\star=\alpha^\dagger/2$\\[-1pt]
    {\scriptsize\itshape\color{black!60}underconfident}},
    ymin=0,
    ymax=2.2,
    ytick={1,2}
]
\addplot[
    blue!75!black,
    very thick,
    smooth
] coordinates {
    (0.01,0.0001)
    (0.02,0.0006)
    (0.04,0.0046)
    (0.07,0.0216)
    (0.10,0.0564)
    (0.15,0.1608)
    (0.20,0.3254)
    (0.25,0.5440)
    (0.30,0.8034)
    (0.35,1.0857)
    (0.40,1.3692)
    (0.45,1.6306)
    (0.50,1.8462)
    (0.55,1.9937)
    (0.60,2.0535)
    (0.65,2.0111)
    (0.70,1.8593)
    (0.75,1.6005)
    (0.80,1.2510)
    (0.85,0.8445)
    (0.90,0.4382)
    (0.93,0.2292)
    (0.96,0.0755)
    (0.98,0.0171)
    (0.99,0.0035)
};

\addplot+[
    red!75!black,
    thick,
    ycomb,
    mark=square*,
    mark size=1.7pt
] coordinates {
    (0.0,0.0727)
    (0.2,0.4836)
    (0.4,1.2872)
    (0.6,1.7131)
    (0.8,1.1400)
    (1.0,0.3035)
};

\nextgroupplot[
    title={Bayes-optimal\\
    $\alpha=\alpha^\dagger$\\[-1pt]
    {\scriptsize\itshape\color{black!60}calibrated}},
    ymin=0,
    ymax=2.5,
    ytick={1,2},
    axis on top
]
\fill[yellow!18]
    (axis cs:0,0)
    rectangle
    (axis cs:1,2.5);

\addplot[
    blue!75!black,
    very thick,
    smooth
] coordinates {
    (0.01,0.3640)
    (0.02,0.4060)
    (0.04,0.4583)
    (0.07,0.5115)
    (0.10,0.5529)
    (0.15,0.6098)
    (0.20,0.6593)
    (0.25,0.7048)
    (0.30,0.7484)
    (0.35,0.7913)
    (0.40,0.8342)
    (0.45,0.8779)
    (0.50,0.9231)
    (0.55,0.9707)
    (0.60,1.0216)
    (0.65,1.0769)
    (0.70,1.1386)
    (0.75,1.2090)
    (0.80,1.2926)
    (0.85,1.3973)
    (0.90,1.5413)
    (0.93,1.6658)
    (0.96,1.8594)
    (0.98,2.0991)
    (0.99,2.3409)
};

\nextgroupplot[
    title={$\alpha$-TCAV\\
    $\alpha=1.5\alpha^\dagger$\\[-1pt]
    {\scriptsize\itshape\color{black!60}confident}},
    ymin=0,
    ymax=5.5,
    ytick={2,4}
]
\addplot[
    blue!75!black,
    very thick,
    smooth
] coordinates {
    (0.01,1.4881)
    (0.02,1.1486)
    (0.04,0.9040)
    (0.07,0.7603)
    (0.10,0.6900)
    (0.15,0.6291)
    (0.20,0.5986)
    (0.25,0.5834)
    (0.30,0.5776)
    (0.35,0.5787)
    (0.40,0.5856)
    (0.45,0.5977)
    (0.50,0.6154)
    (0.55,0.6392)
    (0.60,0.6703)
    (0.65,0.7107)
    (0.70,0.7640)
    (0.75,0.8359)
    (0.80,0.9377)
    (0.85,1.0934)
    (0.90,1.3668)
    (0.93,1.6704)
    (0.96,2.2996)
    (0.98,3.4344)
    (0.99,5.1459)
};

\nextgroupplot[
    title={TCAV\\
    $\alpha\to\infty$\\[-1pt]
    {\scriptsize\itshape\color{black!60}overconfident}},
    ymin=0,
    ymax=4.1,
    ytick={2,4}
]

\addplot[
    blue!75!black,
    very thick,
    dashed,
    smooth
] coordinates {
    (0.01,2.3115)
    (0.02,1.3573)
    (0.04,0.8149)
    (0.07,0.5527)
    (0.10,0.4384)
    (0.15,0.3442)
    (0.20,0.2957)
    (0.25,0.2670)
    (0.30,0.2491)
    (0.35,0.2381)
    (0.40,0.2319)
    (0.45,0.2296)
    (0.50,0.2308)
    (0.55,0.2354)
    (0.60,0.2439)
    (0.65,0.2571)
    (0.70,0.2767)
    (0.75,0.3056)
    (0.80,0.3499)
    (0.85,0.4235)
    (0.90,0.5665)
    (0.93,0.7425)
    (0.96,1.1566)
    (0.98,2.0467)
    (0.99,3.6809)
};

\addplot+[
    orange!85!black,
    ultra thick,
    ycomb,
    mark=*,
    mark size=2pt
] coordinates {
    (0,0.3446)
    (1,0.6554)
};

\end{groupplot}

\coordinate (azero)
    at ([yshift=-1.05cm]alphafig c1r1.south);
\coordinate (astar)
    at ([yshift=-1.05cm]alphafig c2r1.south);
\coordinate (adagger)
    at ([yshift=-1.05cm]alphafig c3r1.south);
\coordinate (apost)
    at ([yshift=-1.05cm]alphafig c4r1.south);
\coordinate (ainfty)
    at ([yshift=-1.05cm]alphafig c5r1.south);

\draw[very thick,->]
    ([xshift=-0.32cm]azero)
    --
    ([xshift=0.32cm]ainfty)
    node[right] {$\alpha$};

\draw[thick]
    ([yshift=2pt]azero)
    --
    ([yshift=-2pt]azero)
    node[below=2pt] {$0$};

\draw[thick]
    ([yshift=2pt]astar)
    --
    ([yshift=-2pt]astar)
    node[below=2pt] {$\alpha^\star$};

\draw[
    very thick,
    yellow!55!black
]
    ([yshift=3pt]adagger)
    --
    ([yshift=-3pt]adagger)
    node[below=2pt,text=black] {$\alpha^\dagger$};

\draw[thick]
    ([yshift=2pt]apost)
    --
    ([yshift=-2pt]apost)
    node[below=2pt] {$1.5\alpha^\dagger$};

\draw[thick]
    ([yshift=2pt]ainfty)
    --
    ([yshift=-2pt]ainfty)
    node[below=2pt] {$\infty$};

\node[anchor=south,yshift=2pt]
    at ($(azero)!0.5!(adagger)$)
    {unimodal regime};

\node[anchor=south,yshift=2pt]
    at ($(adagger)!0.5!(ainfty)$)
    {bimodal regime};

\end{tikzpicture}%
}%

\begingroup
\dimen0=0.16\textwidth
\edef\alphaPanelWidth{\the\dimen0}%
\setbox0=\hbox{\alphaFigureBody}%
\dimen2=\dimexpr\wd0-\dimen0*5\relax
\dimen0=\dimexpr(\textwidth-\dimen2)/5\relax
\edef\alphaPanelWidth{\the\dimen0}%
\setbox0=\hbox{\alphaFigureBody}%
\resizebox{0.88\textwidth}{!}{\copy0}%
\endgroup

\caption{Influence of $\alpha$ on the distribution of $\alpha$-TCAV for
$\sqrt{N}\mu/\sigma=0.4$ and $s=5$, so that
$\alpha^\star=\alpha^\dagger/2$. Writing
$a:=\alpha/\alpha^\dagger$, we have
$\alpha S_N\sim\mathcal{N}(0.4a,a^2)$; blue curves show the corresponding
exact Probit--Normal densities. As $\alpha\to0$, the score converges to
$1/2$. At $\alpha^\star$, its mean matches Multi-TCAV, whose normalized
Binomial law $s\,\P(T_{\mathrm{multi}}=k/s)$ is shown by red stems.
The choice $\alpha^\dagger$ marks the unimodal-to-bimodal transition:
its density is monotone for $\mu>0$ and uniform for $\mu=0$. For larger
$\alpha$, mass moves toward the boundaries, approaching the Bernoulli
TCAV limit shown in orange; the dashed curve uses
$\alpha=4\alpha^\dagger$.}
\label{fig:alpha_influence_distributions}
\end{figure*}

\vspace{-0.2cm}
\paragraph{Choice of $\alpha$.} We are left with the choice of the parameter in $\alpha$-TCAV, which provides an entire parameterized family of TCAV scores. Let us begin our investigation by observing the extreme cases of $\alpha \to 0$, flattening out the distribution and becoming indifferent, and $\alpha \to \infty$, when we recover TCAV; recall Fig. \ref{fig:probit_indicator_fct}. It turns out that we can recover Multi-TCAV as well:
the means of the corresponding random variables coincide exactly when 
$\Phi ( \sqrt{n}\mu/\sigma ) = \Phi (\mu\rho_\alpha)$, \ie when 
\begin{equation}
\alpha^\star
:=
\frac{1}{\sigma}
\sqrt{\frac{n}{1-1/s}}
=
\frac{1}{\sigma}
\sqrt{\frac{N}{s-1}}.
\label{eq:alpha_star_main_part}
\end{equation}
As we show in the appendix, the variance reduces by roughly $1/3$ when using $\alpha^\star$-TCAV, compared to standard Multi-TCAV.
Still, while $\alpha^\star$ recovers Multi-TCAV mean at a reduced variance, it may still not be the optimal choice.
This method computes an arbitrary value in $[0,1]$ that lacks a practical interpretation. In contrast, we suggest a \textit{calibrated TCAV score in $[0,1]$ that corresponds to the posterior probability that the concept contributes to the prediction of the class.} 
Taking a Bayesian perspective (see the appendix for details), 
we obtain the posterior probability
\begin{equation}
\P(\mu>0\mid S_N)
=
\Phi\left(\frac{\sqrt{N}\,S_N}{\sigma}\right)
\label{eq:posterior_sign_prob_main_part}
\end{equation}
that is, the probability of the concept contributing towards the model's decision, given a (deterministic) observation $S_N$.
The $\alpha$-TCAV score $\Phi(\alpha S_N)$ agrees exactly with the probability in \eqref{eq:posterior_sign_prob_main_part} 
when $\alpha = \alpha^\dagger$, where  
\begin{equation}
\alpha^\dagger
:=
\frac{\sqrt{N}}{\sigma}.
\label{eq:alpha_dagger_main_part}
\end{equation}
In particular, $\alpha^\star < \alpha^\dagger$ for $s > 2$. From a fundamental viewpoint, $\alpha = \alpha^\dagger$ is the most principled choice --- this calibration leads to a score more confident (and higher variance) compared to Multi-TCAV, but less confident (and lower variance) compared to  standard TCAV.
Remarkably, the Bayes-optimal value $\alpha^\dagger$ of \eqref{eq:alpha_dagger_main_part} allows a highly elegant mathematical characterization: as illustrated in Fig. \ref{fig:alpha_influence_distributions}, it corresponds exactly to the point of transition of the Probit-Normal distribution between the \textit{unimodal} (Binomial-like; as Multi-TCAV) and the \textit{bimodal} regime (Bernoulli-like; as TCAV), where its density becomes \textit{monotone}. Intuitively, it corresponds to the most confident score (\ie shifting weight towards $1$, or $0$, respectively), just before the distribution switches into the bimodal regime.

\begin{figure*}[t]
    \centering
    \includegraphics[width=0.94\textwidth]{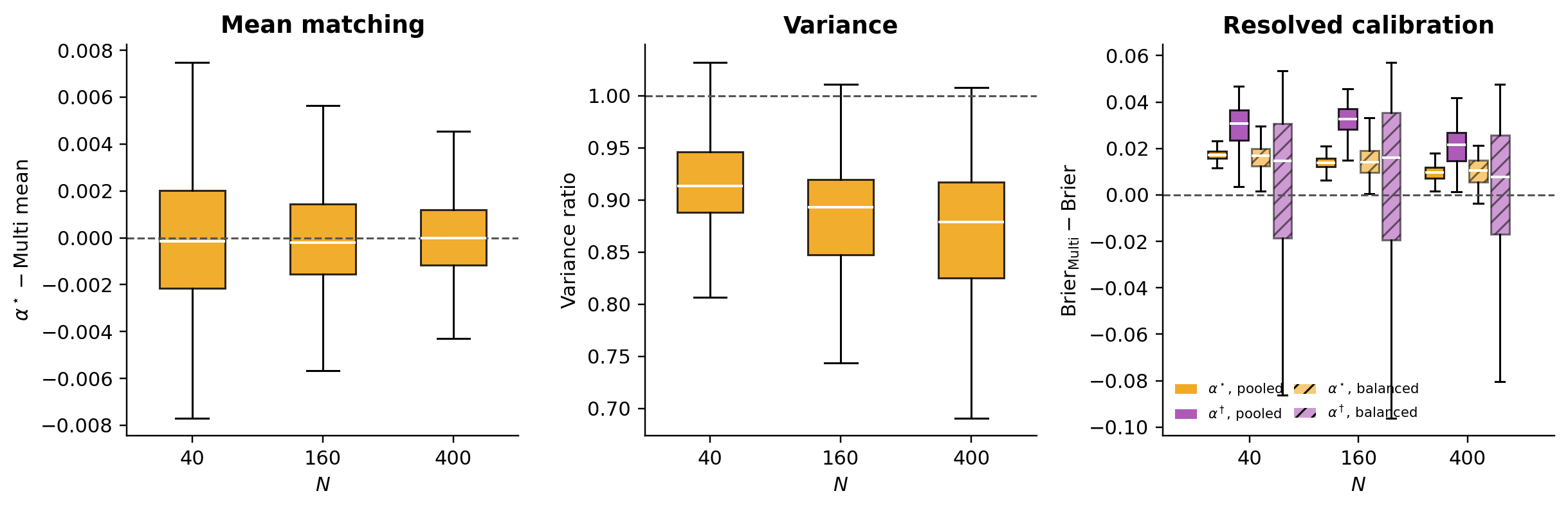}
    \caption{CUB results over three network seeds, ten attributes, ten
    classes, and $250$ paired CAV redraws per setting. (a) Difference
    between the mean scores of $\alpha^\star$-TCAV and Multi-TCAV.
    (b) Their variance ratio. (c) Brier-score reduction relative to
    Multi-TCAV under the species-balanced reference. Positive values in
    panel~(c) indicate an improvement.}
    \label{fig:cub_main_claims}
\end{figure*}

\vspace{-0.2cm}
\paragraph{From Individual Inputs to Classes.}
The Gaussian model \eqref{eq:gaussian_model_main_part} is conditional on a fixed input $\vx$. For a whole class, however, both the mean and variance of the sensitivity score $S_{C,k,l}(\vx)$ from \eqref{eq:TCAV_sensitivity_inner_product} generally depend on $\vx$. The marginal distribution over $\vx\sim\mathcal{D}_k$ is therefore a Gaussian mixture, for which a single class-wide noise parameter $\sigma$ does not generally provide a valid calibration. Instead, we estimate the sensitivity uncertainty for each input by projecting one common CAV covariance onto its gradient. 
Let  us write
\begin{equation*}
    \mathcal{X}_k=\{\vx_1,\ldots,\vx_m\}, \qquad 
    \vz_i:=\nabla h_{l,k}(f_l(\vx_i)),
\end{equation*}
and $\mSigma$ for the CAV covariance (see Sec. \ref{app:sec_cav_distributions}). Also, put
\begin{equation}
    \tau_i^2:=\vz_i^\top\mSigma\vz_i = \Var_{\vw_\CAV} \left(  S_{C,k,l} (\vx_i) \right)
    \label{eq:inputwise_directional_variance_main}
\end{equation}
Note $\mSigma$ is estimated once; only its projection onto $\vz_i$ changes. 
We standardize each sensitivity by its standard deviation and obtain the averaged class score
\begin{align}
        t_i & :=\frac{S_{C,k,l}(\vx_i)}{\tau_i}, \label{eq:t_i} \\
    \TCAV_{C,k,l}^{(\kappa)} & := \frac{1}{m}\sum_{i=1}^{m}\Phi \left( \kappa t_i \right). \label{eq:inputwise_alpha_TCAV_class_main}
\end{align}
Here $\kappa$ is a dimensionless sharpness parameter with the two distinguished choices matching $\alpha^\star$ and $\alpha^\dagger$.
\vspace{-0.2cm}
\begin{itemize}[leftmargin=0.5cm, noitemsep] 
    \item \textbf{Bayesian calibration:}
    Choose $\kappa=1$, corresponding to the input-dependent scale
    $\alpha_i^\dagger=1/\tau_i$.
    \item \textbf{Multi-TCAV matching:}
    $\kappa=1/\sqrt{s-1}$, corresponding to
    $\alpha_i^\star=1/(\sqrt{s-1}\,\tau_i)$ and matching the conditional mean of Multi-TCAV with $s$ batches.
\end{itemize}
\vspace{-0.2cm}
Note that $\kappa\to\infty$ recovers standard TCAV. If $\vz_i=\mathbf{0}$, and hence the standardized sensitivity is undefined, we assign the neutral score $1/2$. A complete derivation, further properties of the standardization, and practical covariance estimators are provided in Sec.~\ref{app:TCAV_distribution_and}.

\paragraph{Class-level Variance and Error Guarantees.}
Beyond individual inputs, the mean-matching choice also yields guarantees for the entire class score in \eqref{eq:inputwise_alpha_TCAV_class_main}.
Consider $\mSigma$ to be known, and, for $s\geq2$,
\begin{equation}
\begin{aligned}
 \vw_\CAV^{(j)}
 &\overset{\mathrm{iid}}{\sim}
 \mathcal N_{d_l}\!\left(\E[\vw_\CAV],s\mSigma\right),
 \quad j=1,\ldots,s,\\
 \vw_\CAV&=\frac1s\sum_{j=1}^s\vw_\CAV^{(j)}.
\end{aligned}
\label{eq:rb_main_joint_model}
\end{equation}
Further, we assume $\tau_i > 0$ for the variances in \eqref{eq:inputwise_directional_variance_main}.
At the mean-matching choice $\kappa^\star$ in \eqref{eq:kappa_distinguished_choices_app}, Thm.~\ref{thm:class_level_rb} gives
\begin{equation}
 \TCAV_{C,k,l}^{(\kappa^\star)}
 =\E \left[ T_{\mathrm{multi}}\mid\vw_\CAV \right].
 \label{eq:rb_main_identity}
\end{equation}
The class score preserves Multi-TCAV's expectation and has no larger variance or mean-squared error relative to any fixed
class-level target, including the noise-free positive-sensitivity fraction, while preserving its bias.
These guarantees account for all dependence between inputs induced by the shared CAV; they do not follow from pointwise variance
comparisons alone. The identity shows that $\alpha^\star$-TCAV averages out the remaining batch fluctuations given the pooled CAV,
without requiring the individual batch CAVs to be computed.

\vspace{-0.2cm}
  
\paragraph{Consequences.}
Under the fixed-input Gaussian model, the family of $\alpha$-TCAV scores admits a classical bias--variance interpretation. For fixed $\mu\neq0$, regard $Y_\alpha=\Phi(\alpha S_N)$ as an estimator of the noise-free sign target $\theta(\mu):=\mathds{1}_{\{\mu>0\}}$ underlying TCAV. Here, $\mu$ and $\theta(\mu)$ are fixed, while the CAV training data induce the randomness of $S_N$, and hence of $Y_\alpha$. Throughout this paragraph, expectations and variances are taken with respect to the randomness of the CAV. We put
\begin{equation}
\rho_\alpha:=\frac{\alpha}{\sqrt{1+\alpha^2\sigma^2/N}},
\qquad
c_\alpha:=\frac{\alpha^2\sigma^2/N}{1+\alpha^2\sigma^2/N}.
\label{eq:alpha_tcav_moment_parameters}
\end{equation}
Now \eqref{eq:probit_normal_mean} and \eqref{eq:probit_normal_variance}, writing $V_\alpha(\mu):=\Var(Y_\alpha)$, give us
\begin{align}
\operatorname{Bias}_\alpha(\mu)
:=&\E[Y_\alpha]-\theta(\mu)
=\Phi(\rho_\alpha\mu)-\theta(\mu), \label{eq:alpha_tcav_bias}\\
V_\alpha(\mu) 
=&\Phi_2(\rho_\alpha\mu,\rho_\alpha\mu;c_\alpha)
-\Phi^2(\rho_\alpha\mu),  \label{eq:alpha_tcav_variance}\\
\operatorname{MSE}_\alpha(\mu)
:=&\E\!\left[(Y_\alpha-\theta(\mu))^2\right]
\label{eq:alpha_tcav_bias_variance_mse}\\
=&\operatorname{Bias}_\alpha(\mu)^2+V_\alpha(\mu).
\nonumber
\end{align}
A small value of $\alpha$ keeps the score close to $1/2$, with low variance but substantial bias relative to the binary target. Increasing $\alpha$ reduces this absolute bias and moves the score toward standard TCAV, while also making it more sensitive to CAV-estimation noise near zero. The resulting variance and MSE need not vary monotonically with $\alpha$, so no single value is pointwise optimal for every signal strength $\mu$.
This viewpoint also clarifies the effect of splitting a fixed sampling budget. Multi-TCAV shifts the expected score toward $1/2$ and is therefore more biased than standard TCAV for every fixed $\mu\neq0$ and $s>1$. The choice $\alpha^\star$ in \eqref{eq:alpha_star_main_part} reproduces the Multi-TCAV mean, and hence its pointwise bias, with no larger variance. By contrast, $Y_{\alpha^\dagger}$ from \eqref{eq:alpha_dagger_main_part} is the posterior mean of $\theta(\mu)$ under the flat prior and squared loss.   Figure~\ref{fig:alpha_influence_distributions} illustrates this bias--variance continuum qualitatively; formal limits and additional details are provided in Appendix~\ref{app:alpha_tcav_bias_variance}.

\vspace{-0.2cm}

\paragraph{Practical Recommendations.}
Our results suggest the following guidelines:
\vspace{-0.2cm}
\begin{itemize}[leftmargin=0.5cm, noitemsep]
\item \textbf{Use the full sampling budget.} Estimate one CAV and its covariance from all available concept and reference examples.
\item \textbf{Calibrate each input separately.} Use $t_i=S_{C,k,l}(\vx_i)/\tau_i$ from \eqref{eq:t_i}; across-input variation does not measure CAV-estimation uncertainty.
\item \textbf{Choose the scale by its purpose.} Use $\kappa^\star=1/\sqrt{s-1}$ to match the mean of Multi-TCAV with no larger variance under the model, and $\kappa^\dagger=1$ for posterior sign probabilities.
\end{itemize}


\clearpage
\subsection*{AI use statement}
In this work, we used generative AI tools for mainly for code generation and support in writing the paper, like drafting text and generating figures for illustration. To a lesser extent, AI tools have been employed to verify mathematical derivations. 

\begingroup
\raggedbottom
\setlength{\bibsep}{2pt}
\bibliographystyle{plainnat}
\bibliography{references}
\endgroup

\clearpage
\appendix
\thispagestyle{empty}

\onecolumn

\aistatstitle{$\alpha$-TCAV:  Supplementary Material}

\section{Experimental Results}
\label{sec:experiments}

In this section we provide additional empirical results supporting the
summary in Fig.~\ref{fig:cub_main_claims}. The experiments test the main
theoretical predictions in a realistic setting: mean matching between
$\alpha^\star$-TCAV and Multi-TCAV, variance reduction, probabilistic
calibration. Furthermore, the preservation of concept rankings is investigated, and the decay of
CAV-estimation uncertainty with the sampling budget.

\paragraph{Setup.}
We use \textit{CUB-200-2011} \citep{WahCUB_200_2011} with its official train--test split. ImageNet-pretrained
ResNet-18 representations are kept fixed, and a nonlinear classification
head is fitted on top of these representations. The data used for model
fitting, finite-sample CAV construction, reference-CAV estimation, and
evaluation are kept disjoint. We repeat the complete procedure for three
network seeds and evaluate ten attributes, ten bird classes, and the
budgets $N\in\{40,160,400\}$. At every setting, all methods are evaluated
on the same $250$ paired PatternCAV redraws; Multi-TCAV uses $s=5$ equal
batches. Attributes are selected using support criteria before their TCAV
results are computed, thereby avoiding selection based on the reported
performance measures.

We consider two reference directions. The \emph{pooled reference} is
estimated by taking the difference between all positive and negative
reference examples. Because the prevalence of an attribute varies across
bird species, this direction may partly reflect species composition. The
\emph{species-balanced reference} instead averages within-species
positive--negative contrasts, giving equal weight to the contributing
species and reducing this source of confounding. A reference sign is
called resolved when its estimated margin exceeds $2.576$ reference
standard errors in absolute value.

\paragraph{Evaluation criteria.}
Let $p_{ib}^{(M)}\in[0,1]$ denote the score assigned to evaluation input
$i$ by method $M$ in CAV redraw $b$, and let $y_i\in\{0,1\}$ denote the
corresponding reference-CAV sign. For the set $\mathcal{R}$ of resolved
inputs, the Brier score \citep{glenn1950verification} is

\begin{equation}
\operatorname{BS}_{b}(M;\mathcal{R})
:=
\frac{1}{|\mathcal{R}|}
\sum_{i\in\mathcal{R}}
\left(p_{ib}^{(M)}-y_i\right)^2.
\label{eq:cub_brier_score}
\end{equation}
The reported value averages \eqref{eq:cub_brier_score} over CAV redraws,
attribute--class pairs, and network seeds. Lower values are better: zero
corresponds to perfect predictions, while the neutral prediction $1/2$
has Brier score $1/4$ against a deterministic binary target. Importantly,
this score measures agreement with the selected reference-CAV direction;
it does not by itself establish that the CAV captures a causal semantic
intervention.
For class-level evaluation, the score is first averaged over all inputs
in a class. Its mean-squared error relative to the reference class
fraction is
\begin{equation}
\operatorname{MSE}_{b}^{\mathrm{class}}(M)
:=
\left[
\frac{1}{m}\sum_{i=1}^{m}p_{ib}^{(M)}
-
\frac{1}{m}\sum_{i=1}^{m}y_i
\right]^2.
\label{eq:cub_class_mse}
\end{equation}
We also report the sampling variance of the resulting class score over
the paired CAV redraws. Metrics are first computed within each
attribute--class pair and network seed and are then summarized across
settings.

\paragraph{Overall comparison.}
Figure~\ref{fig:cub_full_comparison_app} summarizes the four methods.
The class-score variance decreases with $N$ for every method.
$\alpha^\star$-TCAV has the lowest variance and closely tracks
Multi-TCAV in class-level MSE, as predicted by its mean-matching
construction. The Bayes-calibrated choice $\alpha^\dagger$ is sharper:
it performs particularly well against the pooled reference, whereas
$\alpha^\star$ is more robust under the species-balanced target. Standard
TCAV is consistently the most variable because it converts every sign
change into a binary change of the score.

The modest increase of balanced class MSE for some smoothed methods as
$N$ grows is not a failure of the $N^{-1}$ variance law. As stochastic
CAV uncertainty decreases, the systematic difference between the pooled
CAV direction and the species-balanced evaluation target becomes more
visible. This separates uncertainty about an estimated CAV from
misspecification of the direction being estimated.

\begin{figure}[tbp]
    \centering
    \includegraphics[width=0.92\linewidth]
    {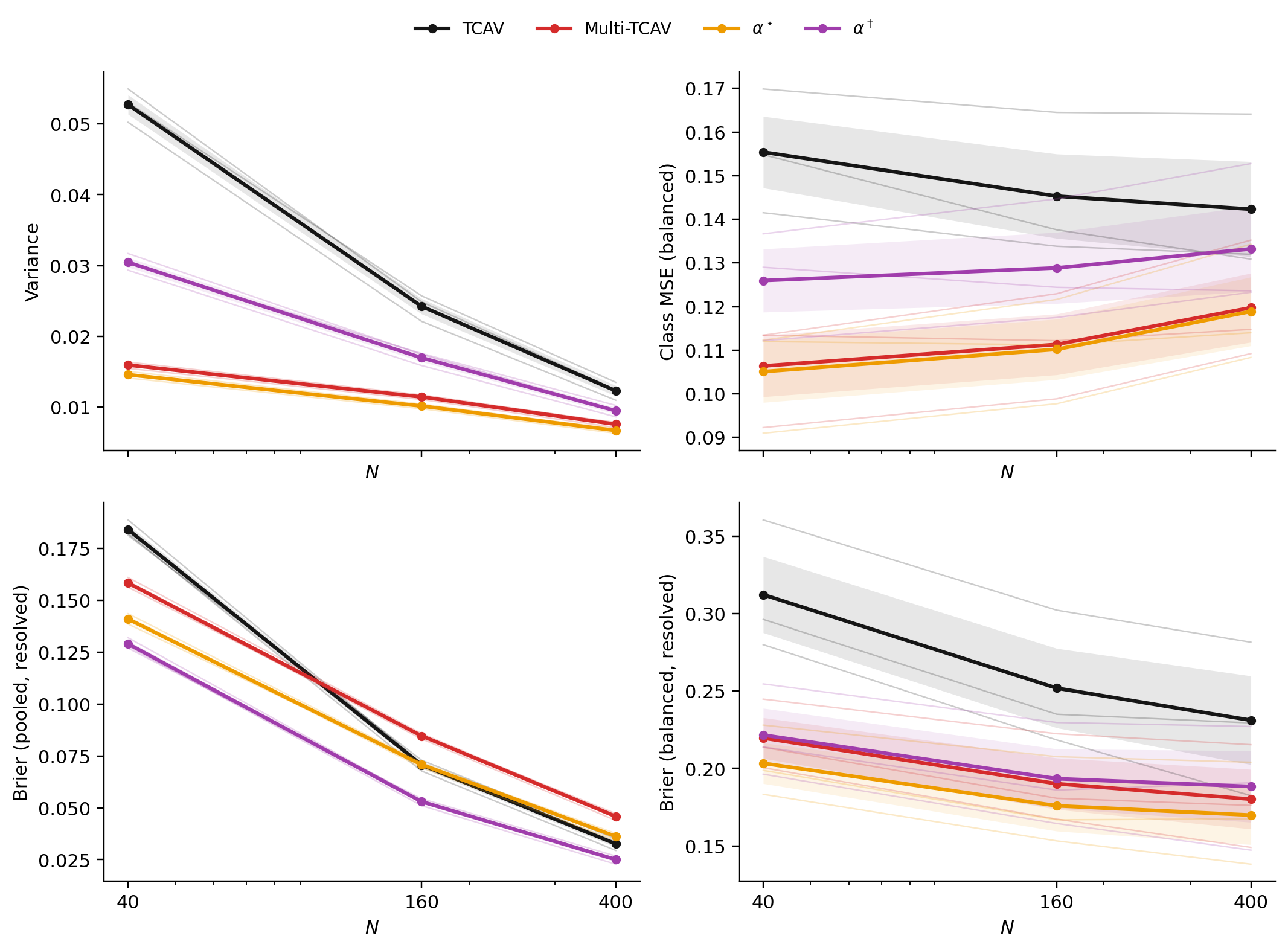}
    \caption{CUB comparison over sampling budgets $N$. Curves show averages
    across attribute--class pairs and network seeds; shaded bands summarize
    their variation. Top: class-score variance and MSE relative to the
    species-balanced reference. Bottom: input-level Brier scores under the
    pooled and species-balanced references.}
    \label{fig:cub_full_comparison_app}
\end{figure}

\paragraph{Calibration.}
We compare predicted scores with observed reference-sign
frequencies. Note that a perfectly calibrated score follows the diagonal. Under
the pooled reference, $\alpha^\dagger$ achieves the strongest aggregate
Brier performance, consistent with its posterior interpretation under
the Gaussian model. Under the species-balanced reference, the curves
shift because the prediction and evaluation directions are no longer
identical; $\alpha^\star$ is then generally the more conservative and
robust choice. This comparison emphasizes that probabilistic calibration
is always relative to a specified target. Smoothing can account for
CAV-estimation uncertainty, but it cannot remove semantic or
species-composition confounding from the target direction.

\begin{figure}[tbp]
    \centering
    \includegraphics[width=0.92\linewidth]
    {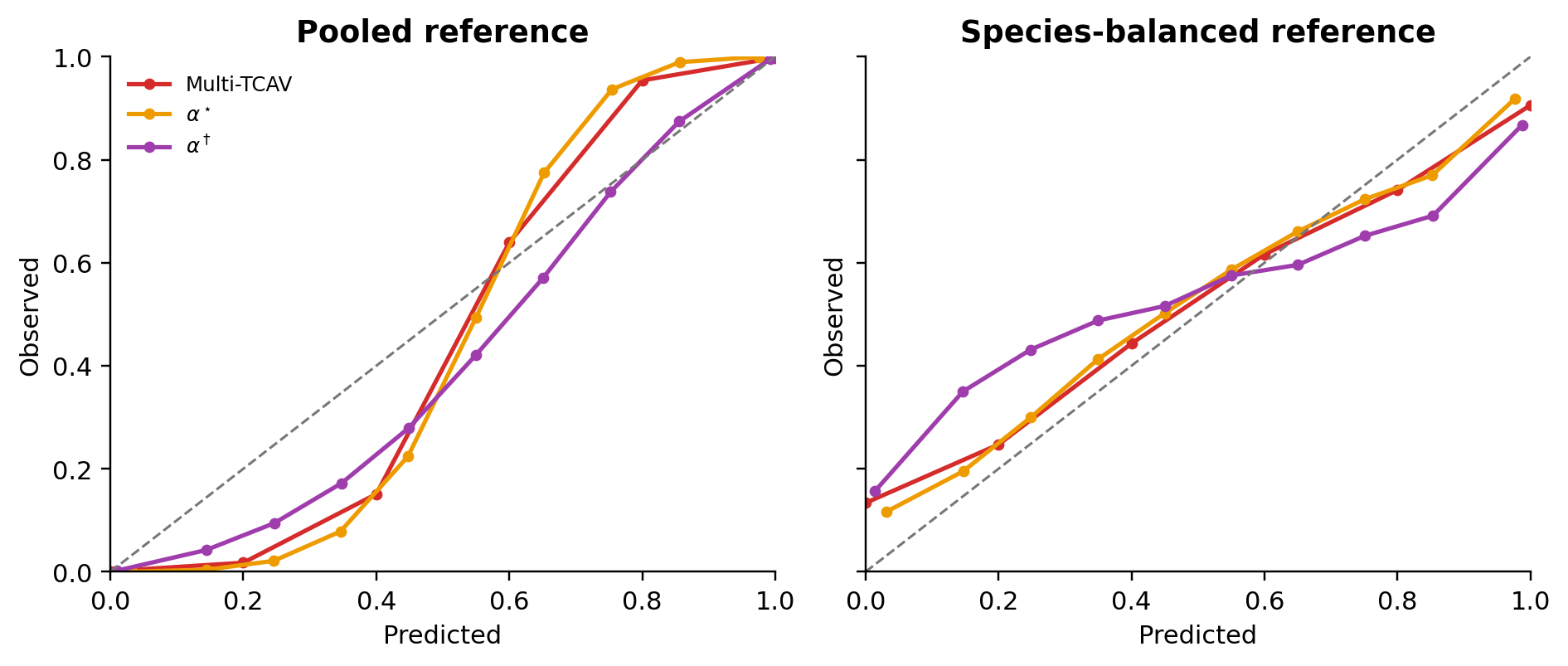}
    \caption{Reliability diagrams for the pooled and species-balanced
    reference signs. The dashed diagonal denotes equality between predicted
    scores and observed fractions of positive reference signs. The two panels
    show how this agreement depends on the chosen reference CAV.}
    \label{fig:cub_reliability_app}
\end{figure}

\paragraph{Concept rankings and reference sensitivity.}
Figure~\ref{fig:cub_rankings_reference_app} shows that smoothing has
little effect on aggregate concept rankings: the rank correlations of
$\alpha^\star$-TCAV with Multi-TCAV and of $\alpha^\dagger$-TCAV with
standard TCAV are both close to one. Thus, the principal effect of the
framework is to change the uncertainty and probabilistic interpretation
of the scores rather than to reorder well-separated concepts.

By contrast, the pooled and species-balanced reference fractions can
differ considerably. The scatter therefore provides direct evidence that
the reference construction is a separate modeling choice. Agreement with
a high-precision pooled CAV should not automatically be interpreted as
agreement with a species-balanced or counterfactual semantic target.

\begin{figure}[tbp]
    \centering
    \includegraphics[width=0.92\linewidth]{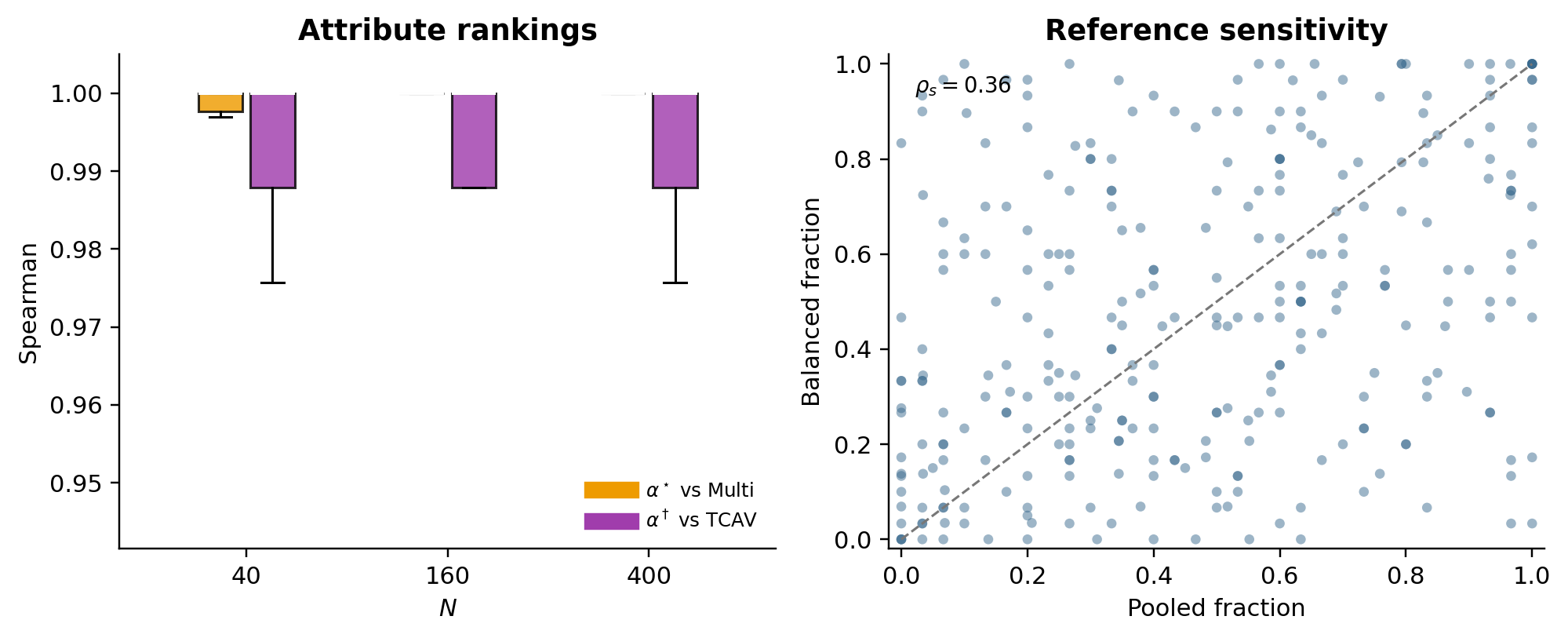}
    \caption{Left: Spearman rank agreement between the matched method
    pairs across sampling budgets. Right: pooled and species-balanced
    reference-positive fractions for the evaluated attribute--class
    combinations. The dashed line denotes equality.}
    \label{fig:cub_rankings_reference_app}
\end{figure}

\paragraph{Variance scaling.}
The left panel of Fig.~\ref{fig:cub_variance_scaling_app} confirms that
the directional CAV variance decreases approximately as $N^{-1}$, in
agreement with the covariance model used in the theoretical analysis.
The class-score variance generally decays more slowly. This is expected, as
the nonlinear Probit or threshold transformation can saturate, and all
input-level scores depend on the same random CAV. Consequently,
input-level terms cannot be treated as independent when estimating the
uncertainty of their class average.

\begin{figure}[tbp]
    \centering
    \includegraphics[width=0.92\linewidth]
    {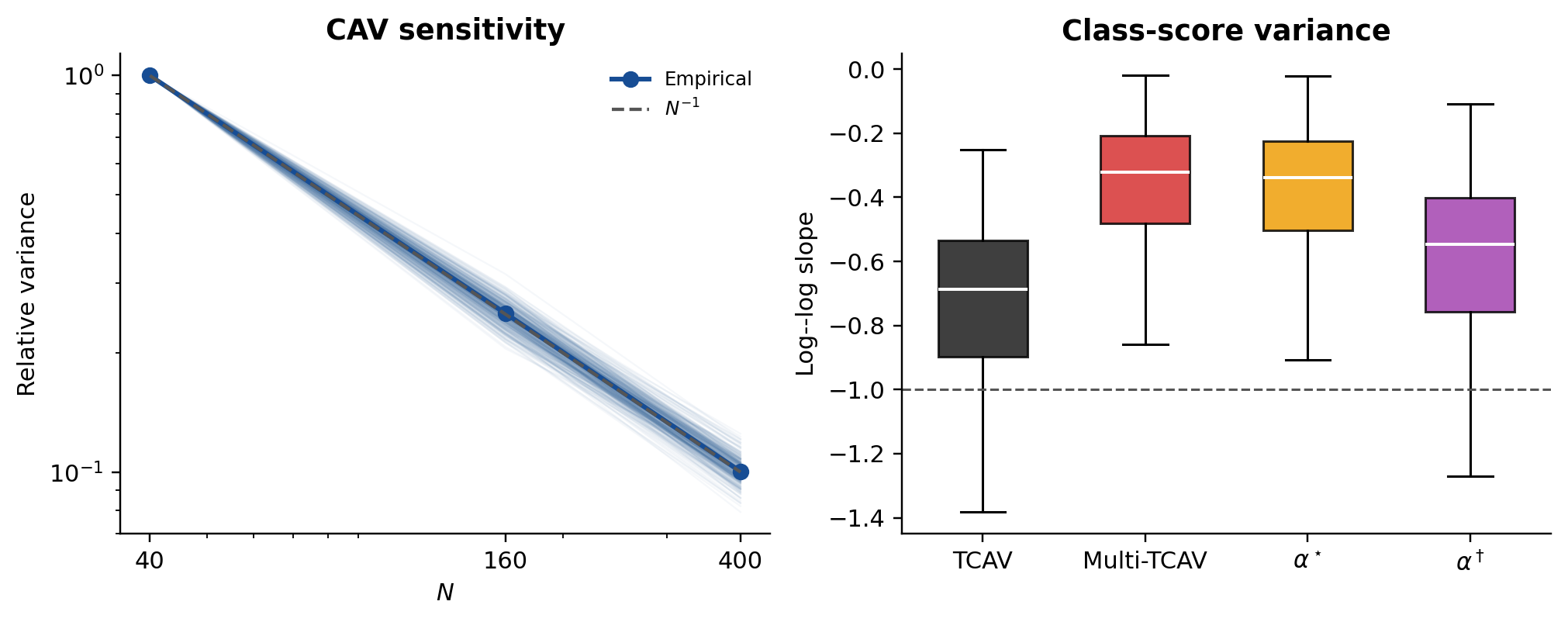}
    \caption{Variance scaling on CUB. Left: empirical directional CAV
    variance relative to its value at $N=40$, with the $N^{-1}$ reference.
    Right: distributions of fitted log--log slopes for the class-score
    variance. The slower class-level decay reflects nonlinear scoring,
    saturation, and dependence induced by the shared CAV.}
    \label{fig:cub_variance_scaling_app}
\end{figure}

Overall, the experiments support the main theoretical predictions:
$\alpha^\star$ closely reproduces the mean of Multi-TCAV while reducing
variance, $\alpha^\dagger$ provides well-calibrated probabilities when
the Gaussian reference model is appropriate, and the inputwise
standardization largely preserves qualitative concept rankings.

\section{Tools from Probability Theory}
\label{sec:probability}

\paragraph{The Probit-Normal Distribution.}
We briefly recall the univariate and bivariate standard normal cdfs. Let $X\sim\mathcal{N}(0,1)$, and let $(X_1,X_2)$ follow a bivariate standard normal distribution with $\operatorname{Corr}(X_1,X_2)=\rho$, where $\rho\in(-1,1)$. We write
\begin{align}
\P(X\leq t)
&=\Phi(t)
=\int_{-\infty}^{t}
\frac{1}{\sqrt{2\pi}}
\exp\left(-\frac{x^2}{2}\right)
\,\mathrm{d}x,
\label{eq:cdf_gaussian}
\\
\P(X_1\leq t_1,\,X_2\leq t_2)
&=\Phi_2(t_1,t_2;\rho)
\nonumber\\
&=\int_{-\infty}^{t_1}\int_{-\infty}^{t_2}
\frac{1}{2\pi\sqrt{1-\rho^2}}
\exp\left(
-\frac{x^2-2\rho xy+y^2}{2(1-\rho^2)}
\right)
\,\mathrm{d}y\,\mathrm{d}x.
\label{eq:cdf_bivariate_gaussian}
\end{align}
For a Gaussian random variable $Z\sim\mathcal{N}(\mu,v)$ and a scaling parameter $\alpha>0$, consider
\begin{equation}
    Y_\alpha:=\Phi(\alpha Z).
    \label{eq:def_scaled_probit_normal}
\end{equation}
Then $Y_\alpha$ follows a scaled \emph{Probit-Normal distribution} on $(0,1)$, equivalently characterized by
$\Phi^{-1}(Y_\alpha)=\alpha Z\sim\mathcal{N}(\alpha\mu,\alpha^2v)$; related parameterizations are also known as Vasicek distributions \citep{FreyMcNeil2003,Meyer2009}. Its mean and variance are given as the following analytic expressions (still requiring numerical integration in practice)
\begin{align}
\E[Y_\alpha]
&=
\Phi\left(
\frac{\alpha\mu}{\sqrt{1+\alpha^2v}}
\right),
\label{eq:probit_normal_mean}
\\
\Var(Y_\alpha)
&=
\Phi_2\left(
\frac{\alpha\mu}{\sqrt{1+\alpha^2v}},
\frac{\alpha\mu}{\sqrt{1+\alpha^2v}};
\frac{\alpha^2v}{1+\alpha^2v}
\right)
-
\Phi^2\left(
\frac{\alpha\mu}{\sqrt{1+\alpha^2v}}
\right).
\label{eq:probit_normal_variance}
\end{align}
The parameter $\alpha$ continuously interpolates between a neutral constant score and a hard decision. More precisely,
\begin{equation}
Y_\alpha
\xrightarrow[\alpha\to0]{\mathrm{a.s.}}
\frac{1}{2},
\qquad
Y_\alpha
\xrightarrow[\alpha\to\infty]{\mathrm{a.s.}}
\mathds{1}_{\{Z>0\}},
\end{equation}
and hence
\[
\lim_{\alpha\to0}\E[Y_\alpha]=\frac12,
\qquad
\lim_{\alpha\to0}\Var(Y_\alpha)=0,
\]
whereas, with $p=\Phi(\mu/\sqrt{v})$,
\[
\lim_{\alpha\to\infty}\E[Y_\alpha]=p,
\qquad
\lim_{\alpha\to\infty}\Var(Y_\alpha)=p(1-p).
\]
The shape of the Probit-Normal density is determined by the latent variance $\alpha^2v$: it is unimodal if $\alpha^2v<1$, uniform for $\alpha^2v=1$ and $\mu=0$, monotone for $\alpha^2v=1$ and $\mu\neq0$ (increasing for $\mu>0$ and decreasing for $\mu<0$), and U-shaped or bimodal if $\alpha^2v>1$ \citep{Meyer2009}. In the TCAV model, where $v=\sigma^2/N$, the transition occurs precisely at
\[
\alpha^\dagger=\frac{\sqrt{N}}{\sigma}.
\]

\paragraph{Variances of Averages of Random Variables and Popoviciu's Inequality.} Let us briefly recall the following very basic fact for the variance of the average of i.i.d. random variables, which turns out useful e.g. in the context of Multi-TCAV in Sec. \ref{subsec:sensitivity_TCAV_and_multi_run_TCAV} as it describes the decay of variances with increasing sample size.

\begin{lemma}\label{lem:variance_empirical_average}
Consider a real-valued random variable $X: \Omega \to \R$ on a probability space $(\Omega, \mathcal{E}, \P)$ with $\Var(X) < \infty$. Let $X_1, X_2, \dots, X_N$ be $N$ i.i.d. copies of $X$. 
Let $S_N$ be the empirical average
\begin{equation}
    S_N = \frac{1}{N} \sum_{i=1}^N X_i.
    \label{eq:empirical_average}
\end{equation}
Then, $\Var(S_N) = \Var(X)/N = \mathcal{O}\left(1/N\right)$.
\end{lemma}


Let us furthermore recall the classical \textit{Popoviciu's Inequality} \citep{popoviciu1935equations}
from probability theory, which provides a general upper bound for the variance of bounded random variables. 

\begin{theorem}[Popoviciu's Inequality]
\label{thm:popoviciu}
Let $X$ be an (essentially) bounded random variable with $a \leq X \leq b$ almost surely for some $a,b \in \R$ with $a < b$.
Then, the variance of $X$ is bounded as follows by
\begin{equation}
\Var(X)  \leq \frac{1}{4}(b - a)^2 .    
\label{eq:popoviciu}
\end{equation}
Equality in \eqref{eq:popoviciu} holds if and only if $\P(X=a)=\P(X=b)=\tfrac{1}{2}$, \ie when the probability mass is evenly distributed between the extreme values. 
\end{theorem}

If $[a,b] = [0,1]$, the upper bound is $\tfrac{1}{4}$, and it is attained for a symmetric Bernoulli distribution; see Figure \ref{fig:variance_illustration}. Therefore, this is is trivially
also the generic upper bound \eqref{eq:variance_TCAV_trivial_upper_bound} for the variance of any TCAV score, understood as random variable taking values in $[0,1]$.

\section{Distribution of TCAV Scores and Numerical Experiments}
\label{app:TCAV_distribution_and}

This section provides the derivations and numerical checks underlying the distributional summary in Table~\ref{table:TCAV_distributions_main_part}.
Starting from the fixed-input Gaussian model \eqref{eq:gaussian_model_main_part}, we derive the distributions and
moments of TCAV, Multi-TCAV, and $\alpha$-TCAV, followed by the distinguished choices $\alpha^\star$ and $\alpha^\dagger$.  
 
\subsection{Modelling different TCAV approaches}
\label{app:simulating_TCAV} 

Recall our model for simulating the sensitivity scores $S_{C,k,l}(\vx) = \langle \nabla h_{l,k}(f_l(\vx)), \vw_\CAV \rangle = \langle \vz, \vw_\CAV \rangle$ 
(for a fixed but arbitrary input $\vx$ with $\vz = \nabla h_{l,k}(f_l(\vx)) \neq \mathbf{0}$) based on their asymptotically normal distribution  and their variance decay as described in Section \ref{sec:cavs}. Let $X \sim \mathcal{N}(\mu, \sigma^2)$ and i.i.d. copies $X_1, \dots X_{N}$ defined on the same probability space $(\Omega, \mathcal{E}, \P)$. Assume we split it up into $s$ batches of size $n$ (\ie 
with $s \cdot n = N$), we then write
\begin{align}
        X_1, \dots, X_{N} 
&=      X_1^{(1)}, \dots, X_{n}^{(1)},\,  X_1^{(2)}, \dots, X_{n}^{(2)}, \dots\dots\dots, X_1^{(s)}, \dots, X_{n}^{(s)}, \label{eq:batches_sim_Multi_TCAV} \\
        N 
&=      s \cdot n. \label{eq:N_s_n}
\end{align}
By averaging them (all, or batch-wise), we can model the variance decay of the different TCAV approaches. By properties of the normal distribution (Lemma \ref{lem:variance_empirical_average}), we easily obtain the distributions of their averages, \ie 
\begin{equation}
\frac{1}{N}\sum_{j=1}^{N} X_j \sim \mathcal{N}(\mu, \sigma^2/N),
\qquad \mathrm{and} \qquad
\frac{1}{n}\sum_{j=1}^{n} X_j^{(r)} \sim \mathcal{N}(\mu, \sigma^2/n) 
\quad \forall r = 1, \dots, s,
\label{eq:distribution_gaussian_averages}
\end{equation}
for the overall and batch-wise average. We use this approach for modeling the different TCAV
approaches analyzed in this paper. More precisely, it is based firstly on assuming that 
$\langle \vz, \vw_\CAV \rangle$ is  normally distributed for deterministic fixed $\vz$ and random vector $\vw_\CAV$ (playing the role of the CAV in the sensitivity score $S_{C,k,l} (\vx) = \langle \nabla h_{l,k} (f_l (\vx)), \vw_\CAV \rangle$). This assumption is well-justified by standard \textit{CLT}-based arguments indicating an (asymptotically) normal distribution
as pointed out in \citep[Theorem 1 \& Corollary 1]{wenkmann2025variability}; see also Sec. \ref{app:sec_cav_distributions}. 
Secondly, it is a convenient way to model the appropriate variance decay of the random variable 
$\langle \vz, \vw_\CAV \rangle$. Recalling Sec. \ref{sec:cavs} and notably 
Def. \ref{def:variance_of_CAV} of $\Var(\vw_\CAV)  =  \Tr(\Cov(\vw_\CAV))$, and this variance consistently decays at a rate of $\mathcal{O}(N^{-1})$ when using $N$ i.i.d. samples to compute the CAV. Note that \citep{wenkmann2025variability} considers $N$ to be the number of non-concept samples ($N=n_1$ in our notation) and the concept data to be fixed; nevertheless, while this is plausible as non-concept data is cheaper to generate, this restriction is not necessary and we find the same decay for $N=n_1+n_2$, assuming a distribution over both the concept \textit{and} non-concept data; recall also Rem. \ref{remark:sample_size_asymptotics}. This variance decay of $\vw_\CAV$ naturally carries over to the inner product $\langle \vz, \vw_\CAV \rangle$, since
(with $\lambda_{\max}(\mA)$ denoting the largest eigenvalue of a matrix $\mA$) we have
\begin{align*}
     \Var(\langle \vz, \vw_\CAV \rangle) 
=   \vz^\top \Cov(\vw_\CAV) \vz 
\leq \|\vz\|^2 \Tr(\Cov(\vw_\CAV)) = \mathcal{O}(N^{-1}).   
\end{align*}
Thus, our setup is clearly the appropriate way to model (for any $\vx$) the sensitivity scores
$S_{C,k,l} (\vx) = \langle \nabla h_{l,k} (f_l (\vx)), \vw_\CAV \rangle = \langle \vz, \vw_\CAV \rangle$
using $N$ samples to compute $\vw_\CAV$ (or batch-wise for $s$ batches of size $m$ for modeling several instances $\langle \vz, \vw_\CAV^{(r)} \rangle$ with random CAVs $\vw_\CAV^{(r)}$ for $i = 1, \dots, s$) by Gaussian averages as in \eqref{eq:distribution_gaussian_averages}. This however makes it possible to systematically study from a mathematical viewpoint the properties of the different TCAV methods, leading to interesting novel insights. It turns out that each TCAV approach is associated to a corresponding classical probability distribution, namely the \textit{Bernoulli distribution} for the (standard) TCAV, the \textit{Binomial distribution} (normalized to $\{0, \tfrac{1}{s}, \tfrac{2}{s}, \dots, 1\}$, rather than $\{1, \dots, s\}$)
in case of Multi-TCAV, and finally the \textit{Probit-Normal distribution} for the $\alpha$-$\TCAV$. 
Despite its simplicity, this approach has straightforward consequences to the more general case. For instance, showing that a TCAV method has a lower variance than another method (e.g., Multi-TCAV has a lower variance than TCAV) for any fixed realization of $\vz = \nabla h_{l,k}(f_l(\vx)) \neq \mathbf{0}$, it must also hold in expectation for probability distribution over $\vx$, including the (unknown) distribution of class $\mathcal{X}_k$. Thus, exploiting the Gaussianity of the conditional distributions is a powerful approach that enables us to compare the different TCAV approaches more rigorously.  
We summarize our findings in the table below, followed by a detailed derivation and discussion of our results. Thus, for a CAV $\vw_\CAV = \vw_\CAV(N)$ computed using $N$ random samples the sensitivity scores are modeled as normal distributions
 \begin{equation}
  \langle \vz, \vw_\CAV \rangle \sim \frac{1}{N}\sum_{j=1}^{N} X_j \sim \mathcal{N}(\mu, \sigma^2/N),
  \label{eq:cond_sensitivity_score_distribution_model}
 \end{equation}
This makes sense as we expect the mean to be the same when increasing the sample size,
but the variance to reduce at order of $\mathcal{O}(1/N)$. Modeling this using a sequence/sum makes sense
as we have a one-to-one correspondence with the sample size used for computing a CAV, and fits well in the Multi-TCAV framework. Intuitively, this model describes the following situation where $\mu$ is the \textit{signal} (indicating the concept's relevance) and $\sigma^2/N$ is the \textit{noise} (due to the randomness of the CAV $ \vw_\CAV$).
In particular, it can model the following situations.

\begin{itemize}
    \item For $\mu > 0$, the concept has a positive influence on the prediction of the class. With higher probability $\langle \vz, \vw_\CAV \rangle$ are positive, but due to 
    the randomness of $\vw_\CAV$, with lower probability (increasing as $N$ grows) some will be negative.
    \item For $\mu < 0$, the concept has a negative influence on the prediction of the class. With higher probability $\langle \vz, \vw_\CAV \rangle$ are negative, but due to 
    the randomness of $\vw_\CAV$, with lower probability (increasing as $N$ grows) some will be positive.
    \item For $\mu = 0$, the concept is neutral regarding the prediction of the class. The scores $\langle \vz, \vw_\CAV \rangle$ are concentrated around the origin, but due to 
    the randomness of $\vw_\CAV$, it fluctuates around zero (however, it becomes more concentrated as $N$ grows).
\end{itemize}

Later we compare all three situations in experiments in Fig. \ref{fig:synthethic_data_TCAV}.
Before that, let us elaborate on the different properties of the TCAV approaches.
Recall that $\Phi$ is the \textit{cumulative distribution function} \eqref{eq:cdf_gaussian} of a random variable $X$ following the  
standard normal distribution, and $\Phi_2$ from \ref{eq:cdf_bivariate_gaussian} is the cdf of a bivariate normal distribution.

\paragraph{TCAV.}
Note that the distribution of interest is a Bernoulli distribution, whose ``success'' probability is related
to the underlying Gaussian distribution $\mathcal{N}(\mu, \sigma^2/N)$. 
Thus, we obtain the for the mean that
 \begin{equation}
     \E \left[\mathds{1}_{\left\{\tfrac{1}{N}\sum_{j=1}^{N} X_j > 0\right\}}\right] 
=    \P \left( \frac{1}{N}\sum_{j=1}^{N} X_j  > 0 \right)
=    \Phi\left( \frac{\mu}{\sigma/\sqrt{N}} \right) 
=    \Phi\left( \frac{\sqrt{N} \mu}{\sigma} \right).
\label{eq:TCAV_expectation}
\end{equation}
Next, recall that  $Y = Y^2$ for any Bernoulli random variable $Y$, such that $\E[Y] = \E[Y^2]$. Therefore,  
\begin{equation}
    \Var \left( \mathds{1}_{\left\{\tfrac{1}{N}\sum_{j=1}^{N} X_j > 0\right\}}  \right)
=   \Phi\left( \frac{\sqrt{N}\mu}{\sigma} \right) \left(1- \Phi\left( \frac{\sqrt{N}  \mu}{\sigma} \right)\right),
\label{eq:TCAV_variance}
\end{equation}
showing the claimed expression for the variance. Let us comment on our findings.

\begin{itemize}[leftmargin=0.5cm]
    \item If $\mu = 0$, both mean $\Phi(0) = \tfrac{1}{2}$ and 
           variance $\Phi(0) (1 - \Phi(0)) = \tfrac{1}{2} (1 - \tfrac{1}{2}) = \tfrac{1}{4}$ are constant.
    \item If $\mu > 0$, then $N \to \infty$ will push $\Phi\left(\sqrt{N} \mu / \sigma \right) \to 1$, and the variance vanishes.
    \item If $\mu < 0$, then $N \to \infty$ will push $\Phi\left(\sqrt{N} \mu / \sigma \right) \to 0$, and the variance vanishes.
\end{itemize}

\paragraph{$T_{\text{multi}}$.} As a (normalized) sum of $s$ i.i.d. Bernoulli random variables, the random variable describing $T_{\text{multi}}$ follows a (normalized) Binomial distribution taking values in $\{0, \tfrac{1}{s}, \tfrac{2}{s}, \dots, 1\}$. Similar to \eqref{eq:TCAV_expectation}, for any batch $r_0 \in \{1, \dots, s\}$ with $n$ samples, we find that

\begin{equation}
    \E \left[ \frac{1}{s}\sum\limits_{i=1}^{s} \mathds{1}_{\left\{\tfrac{1}{n}\sum_{j=1}^{n} X_j^{(r)} > 0\right\}} \right] = \E \left[ \mathds{1}_{\left\{\tfrac{1}{n}\sum_{j=1}^{n} X_j^{(i)} > 0\right\}} \right] = \P \left( \frac{1}{n}\sum_{j=1}^{n} X_j^{(r_0)}  > 0 \right) = \Phi\left( \frac{\sqrt{n} \mu}{\sigma} \right). \label{eq:TCAV_multi_expectation}
\end{equation}

For a ``success'' probability $p$ (\ie, the probability of $1$, which is also the expectation of the Bernoulli distribution), the variance of a general Bernoulli distribution is $p(1-p)$. Taking the normalization into account and due to the independence, by Lemma \ref{lem:variance_empirical_average}
and using \eqref{eq:TCAV_multi_expectation}, we find immediately that 
\begin{equation}
    \Var \left(
\frac{1}{s}\sum\limits_{i=1}^{s} \mathds{1}_{\left\{\tfrac{1}{n}\sum_{j=1}^{n} X_j^{(i)} > 0\right\}}   
\right)
=
\frac{\Phi\left( \frac{\sqrt{n}\mu}{\sigma} \right) \left(1- \Phi\left( \frac{\sqrt{n}  \mu}{\sigma} \right)\right)}{s}.
\label{eq:TCAV_multi_variance}
\end{equation}

\begin{itemize}[leftmargin=0.5cm]
    \item Let us point out that in the special case of $s = 1$ and $n = N$ (a single batch of size $N$),
          $T_{\text{multi}}$ collapses to the (standard) TCAV, consistently with our findings 
          for the corresponding means and variances.
    \item Furthermore, it is worth mentioning that the mean \eqref{eq:TCAV_multi_expectation} of                          $T_{\text{multi}}$ is independent of $s$ and controlled by $n$, the number of samples per             
          batch (used for one of the CAV computations). 
    \item In contrast, the expression \eqref{eq:TCAV_multi_variance} for the variance is mainly governed by $s$
          and scales like $\mathcal{O}(1/s)$; while the numerator depends on 
          $n$, it is lower bounded by $0$ and upper bounded by $0.25$.
          This has interesting consequences in practice, as the user faces some tradeoff when given 
          a limited sampling budget of $N$ samples: On the one hand, as a low variance is desirable, we may want to 
          choose a large number of batches $s$; however (apart from computational aspects), this enforces a smaller batch size $n$ (when $N = s \cdot n$ is fixed) and will result in a smaller mean (less decisive, more tending towards a neutral assessment, perhaps underestimating). On the other hand, a small value of $s$ (in the extreme case, obtaining (standard) TCAV for $s=1$) may lead to a larger variance and, as $n$ becomes larger, a more
          decisive mean closer to the extreme values $1$ (overwhelming concept influence; if $\mu>0$) and $0$ (negligible concept influence; if $\mu<0$). We explicitly state this next.
    \item Compare mean of TCAV and $T_{\text{multi}}$: for $\mu = 0$ they are the same of course, and they are 
          always bounded as $0 \leq \Phi \leq 1$. 
          For $\mu > 0$ note that TCAV mean is pushed more towards one, while as $n \ll N$, $T_{\text{multi}}$ is more moderate.
          Similar for $\mu < 0$ TCAV mean is pushed more towards zero, while as $n \ll N$, $T_{\text{multi}}$ is more moderate. 
    \item Given a large sampling budget, a user can increase $s$, while keeping $n$ fixed, to reduce 
          the variance. By the \textit{CLT}, the distribution will become Gaussian-like for $s \to \infty$.
          However, increasing the number $s$ requires a large sampling budget, and is computationally 
          expensive.
\end{itemize}

\begin{figure}[h]
\centering
\begin{tikzpicture}
\begin{axis}[
    width=9cm,
    height=4.5cm,
    axis lines=left,
    xlabel={$p$},
    ylabel={Variance},
    ylabel style={yshift=-0.9cm,xshift=-0.25cm},
    xmin=0, xmax=1,
    ymin=0, ymax=0.3,
    xtick={0,0.1,0.3,0.5,0.7,0.9,1},
    xticklabels={
        0,
        $p_{\operatorname{Ber}}^{(1)}$,
        $p_{\operatorname{Bin}}^{(1)}$,
        0.5,
        $p_{\operatorname{Bin}}^{(2)}$,
        $p_{\operatorname{Ber}}^{(2)}$,
        1
    },
    ytick={0,0.25},
    axis line style={thick},
    legend style={
        at={(0.70,1.0)},
        anchor=north west,
        font=\footnotesize,
        draw=none,
        fill=none
    },
    legend cell align={left},
    xshift=-1.2cm
]
\addplot[
    domain=0:1,
    samples=100,
    very thick,
    blue!70!black
] {x*(1-x)};
\addlegendentry{Bernoulli}

\addplot[
    domain=0:1,
    samples=100,
    very thick,
    red!70!black
] {x*(1-x)/3};
\addlegendentry{Binomial}

\draw[dashed,red,thin]
    (axis cs:0.5,0) -- (axis cs:0.5,0.25);
\draw[dashed,red,thin]
    (axis cs:0.5,0.25) -- (axis cs:0,0.25);

\draw[thick,blue!70!black]
    (axis cs:0.1,0) -- (axis cs:0.1,{0.1*0.9});
\draw[thick,red!70!black]
    (axis cs:0.3,0) -- (axis cs:0.3,{0.3*0.7/3});

\draw[thick,blue!70!black]
    (axis cs:0.9,0) -- (axis cs:0.9,{0.9*0.1});
\draw[thick,red!70!black]
    (axis cs:0.7,0) -- (axis cs:0.7,{0.7*0.3/3});
\end{axis}

\begin{scope}[yshift=-1.1cm,xshift=-1.2cm]
    \draw[<-,>=Stealth,thick] (0.2,0) -- (3.2,0);
    \node[below,font=\small] at (1.7,0) {concentration};

    \draw[->,>=Stealth,thick] (4.2,0) -- (7.2,0);
    \node[below,font=\small] at (5.7,0) {concentration};
\end{scope}
\end{tikzpicture}
\caption{Comparison of the variance between a Bernoulli distribution and a
normalized Binomial distribution obtained as the average of $s=3$
independent Bernoulli random variables with success probability $p$. Their
variances are $p(1-p)$ and $p(1-p)/s$, respectively. The figure also
illustrates a finding of Table~\ref{table:TCAV_distributions_main_part}: due to the
splitting of the sample budget, we are not comparing variances at the same
value of $p$. More precisely, for TCAV and Multi-TCAV,
$\Phi(\sqrt{n}\mu/\sigma)$ is more neutral, \ie closer to $1/2$, than
$\Phi(\sqrt{N}\mu/\sigma)$ and therefore shifts the success probability
slightly against the variance reduction; compare
$p_{\operatorname{Bin}}^{(1)}$ with $p_{\operatorname{Ber}}^{(1)}$ and
$p_{\operatorname{Bin}}^{(2)}$ with
$p_{\operatorname{Ber}}^{(2)}$. Nevertheless, the averaging yields an
overall variance reduction.}
\label{fig:variance_illustration}
\end{figure}

\paragraph{$\alpha$-TCAV.}
Finally, let us comment on $\alpha$-$\TCAV$, which under our assumptions follows a Probit-Normal distribution, whose mean and variance are available in exact closed form; see \eqref{eq:probit_normal_mean} and \eqref{eq:probit_normal_variance} in Sec.~\ref{sec:probability}. 
On the one hand, by the pointwise convergence of the scaled Gaussian cdf
$\Phi(\alpha x)$ to the Heaviside function as
$\alpha\to\infty$, sufficiently large values of $\alpha$ make
$\alpha$-TCAV resemble standard TCAV: the score becomes highly decisive,
taking values close to either $0$ or $1$, but may consequently exhibit high
variance. On the other hand, as $\alpha\to0$, the scaled Gaussian cdf
flattens out, and the score converges to the indecisive value $1/2$ with
vanishing variance; also recall Figure \ref{fig:alpha_influence_distributions}. 
Thus, varying $\alpha$ continuously interpolates between these two regimes. 
In the following, we investigate two particular choices of $\alpha$.
Indeed, the exact mean of $\alpha$-TCAV satisfies
\begin{equation}
\lim_{\alpha\to\infty}
\Phi\left(
\frac{\alpha\mu}
{\sqrt{1+\alpha^2\sigma^2/N}}
\right)
=
\Phi\left(
\frac{\sqrt{N}\mu}{\sigma}
\right),
\end{equation}
and its variance converges correspondingly to
\begin{equation}
\lim_{\alpha\to\infty}
\Var \left[
\Phi\left(
\alpha\frac{1}{N}\sum_{j=1}^{N}X_j
\right)
\right]
=
\Phi\left(\frac{\sqrt{N}\mu}{\sigma}\right)
\left(
1-\Phi\left(\frac{\sqrt{N}\mu}{\sigma}\right)
\right).
\end{equation}
Thus, both moments converge exactly to those of standard TCAV, confirming
the consistency of the framework without requiring any auxiliary
approximation constants. At the other endpoint, we immediately obtain that
\begin{equation}
\lim_{\alpha\to0}
\E\left[
\Phi\left(
\alpha\frac{1}{N}\sum_{j=1}^{N}X_j
\right)
\right]
=
\frac{1}{2},
\qquad
\lim_{\alpha\to0}
\Var \left[
\Phi\left(
\alpha\frac{1}{N}\sum_{j=1}^{N}X_j
\right)
\right]
=
0.
\end{equation}

Further, the case of a large sample size, when $N\gg\alpha^2\sigma^2$, is instructive. 
In this case, for sufficiently large sample size, the score is essentially controlled by $\alpha$ and the underlying signal $\mu$, 
\begin{equation}
\Phi\left(
\frac{\alpha\mu}
{\sqrt{1+\alpha^2\sigma^2/N}}
\right)
\approx
\Phi\left(\alpha\mu\right).
\label{eq:alpha_TCAV_theoretical_mean_large_N}
\end{equation}

\begin{figure}[h]
    \centering
    \includegraphics[width=0.9\linewidth]{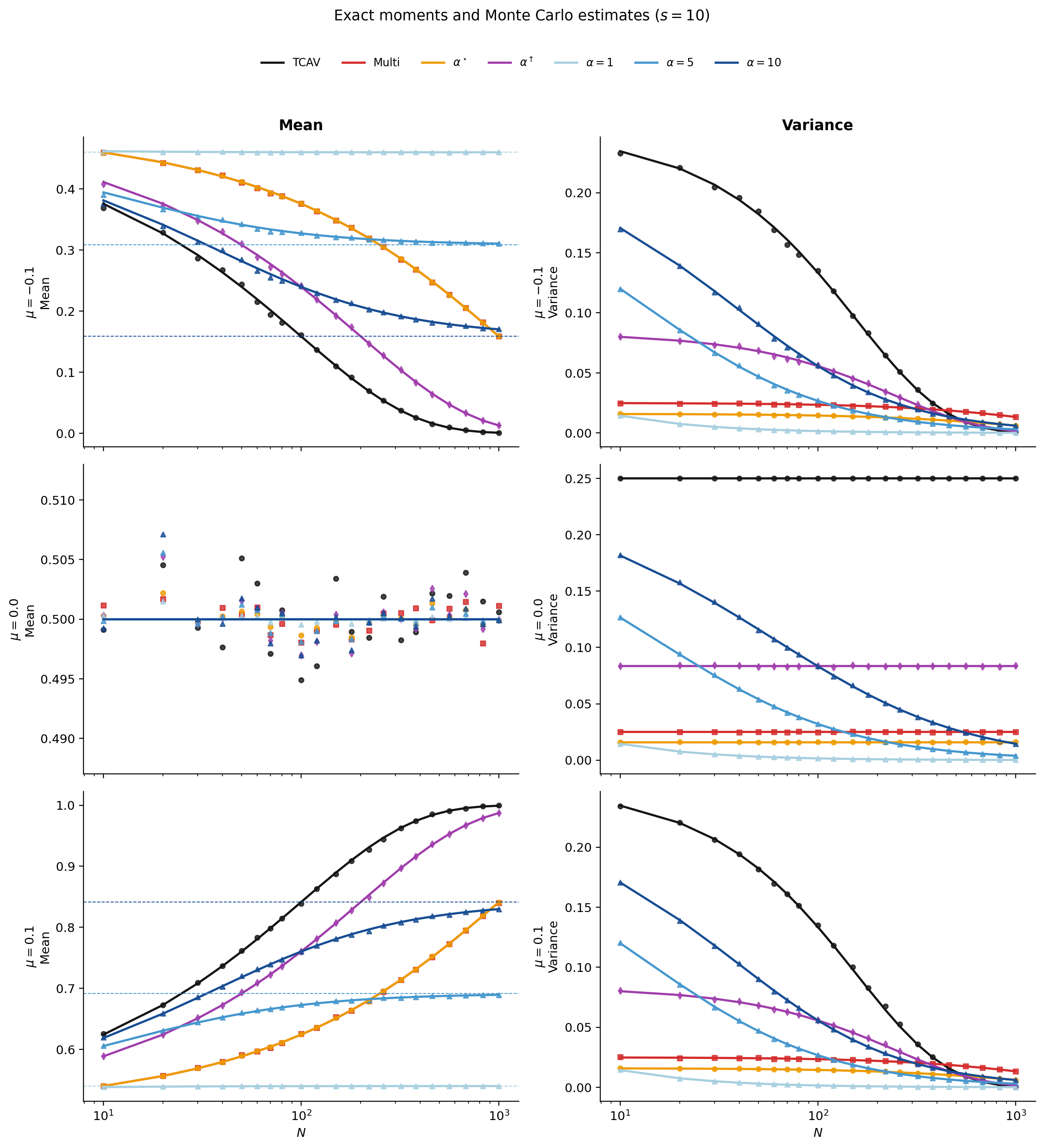}
    \caption{Numerical simulation of \textbf{mean and variance (columns)} of the different TCAV approaches in Table \ref{table:TCAV_distributions_main_part}
    according to the model explained in Sec. \ref{app:TCAV_distribution_and} with $\sigma = 1$ and different values for \textbf{different values of $\mu$ (rows)}
    and $N$ on the $x$-axis (modeling the sample size), showing an excellent agreement between the theoretical prediction and the empirical simulation.
    Standard \textit{(indicator)} TCAV (\textbf{black}) provides the most
    extreme mean (tending towards either $0$ or $1$), while clearly being biased (mean changes with the sample size), and having a large variance 
    (notably large for small sample size), and even non-decaying variance for $\mu = 0$, with a constant variance of $0.25$ as in the worst-case setting of 
    \textit{Popoviciu's Inequality} (Thm. \ref{thm:popoviciu}); also recall \eqref{eq:variance_TCAV_trivial_upper_bound}. 
    Multi-TCAV (\textcolor{red}{red} , with fixed $s = 10$ and $n$ matching each $N = s \cdot n$) can be well-matched using 
    $\alpha^\star$-TCAV \textcolor{yellow}{(yellow)} with $\alpha^\star$ from \eqref{eq:alpha_star} with a significantly lower variance. 
    (Notably, $\alpha^\star$ depends on $N = s \cdot n$: as $s$ is fixed and $N$ grows, so does $n$, and therefore $\alpha^\star$.)
    The Bayes-optimal $\alpha^\dagger$-TCAV with $\alpha^\dagger$ from \eqref{eq:alpha_dagger} is shown in \textcolor{violet}{violet}, highlighting it is 
    \textit{between the over-confident TCAV and the under-confident Multi-TCAV}; it becomes more confident for growing $N$:
    by the interpretation of a probability (rather than some abstract score), the probability approaches zero ($\mu < 0$) or one ($\mu > 0$), respectively, as the sample size grows.
    We plot various other value $\alpha$-TCAV for various values of $\alpha$ in \textcolor{blue}{blue} together with the theoretical means (for large $N$) \ref{eq:alpha_TCAV_theoretical_mean_large_N}. 
    }
    \label{fig:synthethic_data_TCAV}
\end{figure}

\subsection{Choosing $\alpha$ and Normalizing the Sensitivity Scores}
\label{app:normalization}

\paragraph{Matching $\alpha$-TCAV with Multi-TCAV.}
Interestingly, $\alpha$-$\TCAV$ can be used to imitate Multi-TCAV in the
following sense. In this paragraph, we investigate how to choose the parameter $\alpha$
such that the means of $\alpha$-TCAV and Multi-TCAV agree, and compare the corresponding
variances. Thus, for the means to agree we impose that
\begin{equation}
\Phi\left(\frac{\sqrt{n}\mu}{\sigma}\right)
=
\Phi\left(
\frac{\alpha\mu}
{\sqrt{1+\alpha^2\sigma^2/N}}
\right),
\label{eq:alpha_ansatz}
\end{equation}
which trivially holds for $\mu=0$.
As $\Phi$ is a strictly increasing function, for any $\mu \neq 0$ eq. \eqref{eq:alpha_ansatz} is equivalent to
\begin{equation}
\frac{\sqrt{n}}{\sigma}
=
\frac{\alpha}
{\sqrt{1+\alpha^2\sigma^2/N}}.
\end{equation}
Next, using \eqref{eq:N_s_n}, \ie the relation $N = sn$ for batches in Multi-TCAV, and solving for $\alpha$ yields, for
any $s\geq2$,
\begin{equation}
\alpha^\star
=
\frac{1}{\sigma}
\sqrt{\frac{n}{1-1/s}}
=
\frac{1}{\sigma}
\sqrt{\frac{N}{s-1}}.
\label{eq:alpha_star}
\end{equation}
Thus, the means of $\alpha^\star$-TCAV and Multi-TCAV agree. Note that, as $s\downarrow1$, we obtain $\alpha^\star\to\infty$, recovering TCAV, which corresponds to the single-batch case $s=1$. We next compare the variances.
For the choice \eqref{eq:alpha_star}, we have ${\alpha^\star}^2=N/[\sigma^2(s-1)]$, and hence ${\alpha^\star}^2\sigma^2/N=1/(s-1)$. Therefore, the parameter $\alpha^2v / (1+\alpha^2v)$ in the Probit--Normal distribution \eqref{eq:probit_normal_variance} simplifies to

\begin{equation}
\frac{{\alpha^\star}^2\sigma^2/N}{1+{\alpha^\star}^2\sigma^2/N}
=
\frac{1/(s-1)}{1+1/(s-1)}
=
\frac{1/(s-1)}{s/(s-1)}
=
\frac{1}{s}.
\label{eq:alpha_star_probit_correlation}
\end{equation}
Moreover, using $N=sn$ and defining $q:=\sqrt{n}\mu/\sigma$, the location argument in \eqref{eq:probit_normal_variance} becomes
\begin{equation}
\frac{\alpha^\star\mu}{\sqrt{1+{\alpha^\star}^2\sigma^2/N}}
=
\frac{\sqrt{N}\mu}{\sigma\sqrt{s-1}}\sqrt{\frac{s-1}{s}}
=
\frac{\sqrt{N/s}\,\mu}{\sigma}
=
\frac{\sqrt{n}\mu}{\sigma}
=
q.
\label{eq:alpha_star_probit_location}
\end{equation}
Substituting \eqref{eq:alpha_star_probit_correlation} and \eqref{eq:alpha_star_probit_location} into \eqref{eq:probit_normal_variance} gives the variance of $\alpha^\star$-TCAV, which is given by
\begin{equation}
        \Var\left(\alpha^\star\text{-}\TCAV\right) = \Phi_2\left(q,q;\frac{1}{s}\right)-\Phi^2(q)
\end{equation}
Together with the variance of Multi-TCAV, their exact variance ratio is therefore immediately obtained as
\begin{equation}
R(s,q)
:=
s\,
\frac{
\Phi_2\left(q,q;1/s\right)-\Phi^2(q)
}{
\Phi(q)\left(1-\Phi(q)\right)
}.
\label{eq:variance_ratio_simplified}
\end{equation}

\begin{figure}[htbp]
    \centering
    \begin{tikzpicture}
        \begin{axis}[
            axis lines=left,
            xlabel={$s$},
            ylabel={$R(s,0)$},
            xmin=2, xmax=100,
            ymin=0.62, ymax=0.68,
            xtick={2,20,40,60,80,100},
            ytick={0.63,0.64,0.65,0.66,0.67},
            grid=major,
            width=0.9\textwidth,
            height=7cm,
            legend pos=north east,
            legend cell align={left},
            tick label style={font=\small}
        ]
            \fill[blue!15] (axis cs:2,0.63662) rectangle (axis cs:100,0.66667);
            \addlegendimage{area legend,fill=blue!15,draw=none}
            \addlegendentry{exact range}

            \addplot[red,thick,domain=2:100,samples=250] {x*asin(1/x)/90};
            \addlegendentry{$R(s,0)$}

            \addplot[black,dashed,thick,domain=2:100] {0.63662};
            \addlegendentry{limit $2/\pi$}
        \end{axis}
    \end{tikzpicture}
    \caption{Variance ratio for $\mu=0$, decreasing from $R(2,0)=2/3$ to $\lim_{s\to\infty}R(s,0)=2/\pi\approx0.64$.}
    \label{fig:plot_variance_ratio}
\end{figure}

By properties of the Gaussian cdf it can be formally shown that $R(s,q) \leq 1$ for every $s \geq 2$ and $q \in \R$. 
Thus, $\alpha^\star$-TCAV exactly matches the mean of Multi-TCAV while having no larger variance. In the critical regime $\mu=0$, Fig.~\ref{fig:plot_variance_ratio} shows that we obtain a variance reduction of roughly one third.

\paragraph{A Bayesian Interpretation of $\alpha$-TCAV.}
For fixed $\vz$ and random CAV $\vw_\CAV$, whether the concept has a
positive or negative influence on the model's prediction is determined by
the sign of the mean $\mu$, \ie by
\begin{equation}
\theta:=\mathds{1}_{\{\mu>0\}}.
\end{equation} 
Adopting a Bayesian viewpoint, we treat $\mu$ as random and place a
\emph{flat} improper prior $\pi(\mu)\propto1$ on $\R$, encoding no prior belief.  
Recall that a prior is \emph{conjugate} to a
likelihood if the posterior lies in the same family as the prior; for
Gaussian likelihoods with known variance, Gaussian priors --- and flat
priors as their limiting case --- are conjugate; see, e.g.,
\citep[Ch.~2]{gelman2013bayesian}. By Bayes' theorem,
\begin{equation}
        \pi(\mu\mid S_N)   
\propto \exp\left( -\frac{N}{2\sigma^2}(\mu-S_N)^2 \right).
\end{equation}
Therefore, standardizing yields the posterior sign probability of $\mu$, given the observation $S_N$ which is given by
\begin{equation}
\P(\mu>0\mid S_N)
=
\Phi\left(\frac{\sqrt{N}\,S_N}{\sigma}\right),
\label{eq:posterior_sign_prob}
\end{equation} 
the calibrated posterior probability that the concept has a positive influence.
It agrees with $\Phi(\alpha S_N)$ exactly for
\begin{equation}
    \alpha^\dagger := \frac{\sqrt{N}}{\sigma}, \label{eq:alpha_dagger}
\end{equation}
\ie the probability \eqref{eq:posterior_sign_prob} coincides with the score computed by $\alpha^\dagger$-TCAV.
Thus, $\alpha$-TCAV at
$\alpha=\alpha^\dagger$ is exactly the Bayes-optimal estimator of $\theta$
under squared loss
$\ell(\widehat{\theta},\theta)
=(\widehat{\theta}-\theta)^2$
and is calibrated by construction.
Remarkably, the Bayes-calibrated scale is also exactly the phase-transition
point of the Probit-Normal distribution. We put
\begin{equation}
Y_\alpha:=\Phi(\alpha S_N),
\qquad
v_\alpha:=\frac{\alpha^2\sigma^2}{N},
\end{equation}
where $v_\alpha$ is the variance of $\alpha S_N$.
By the change of variables, the density of $Y_\alpha$ on $(0,1)$ is given by
\begin{equation}
f_\alpha(y)
=
\frac{1}{\sqrt{v_\alpha}}
\exp\left[
\frac{v_\alpha-1}{2v_\alpha}
\left(\Phi^{-1}(y)\right)^2
+
\frac{\alpha\mu}{v_\alpha}\Phi^{-1}(y)
-
\frac{\alpha^2\mu^2}{2v_\alpha}
\right].
\label{eq:probit_normal_density_alpha}
\end{equation}
Writing $x=\Phi^{-1}(y)$, differentiation gives
\begin{equation}
\frac{\mathrm{d}}{\mathrm{d}x}
\log f_\alpha\left(\Phi(x)\right)
=
\frac{(v_\alpha-1)x+\alpha\mu}{v_\alpha}.
\end{equation}
Consequently, the density is unimodal when $v_\alpha<1$, uniform for
$v_\alpha=1$ and $\mu=0$, monotone for $v_\alpha=1$ and $\mu\neq0$, and
bimodal when $v_\alpha>1$. Since
\begin{equation}
v_\alpha=1
\quad\Longleftrightarrow\quad
\alpha=\frac{\sqrt{N}}{\sigma}
=\alpha^\dagger,
\end{equation}
the Bayes-calibrated choice occurs precisely at the transition, as
\begin{equation}
\begin{aligned}
\alpha<\alpha^\dagger
&\quad\Longrightarrow\quad
\text{unimodal},\\
\alpha=\alpha^\dagger
&\quad\Longrightarrow\quad
\text{uniform if $\mu=0$, monotone if $\mu\neq0$},\\
\alpha>\alpha^\dagger
&\quad\Longrightarrow\quad
\text{bimodal}.
\end{aligned}
\end{equation}

This recasts $\alpha$-TCAV not only as a smooth approximation, but as a principled posterior estimator. 
The contrast with standard TCAV --- the indicator $\mathds{1}_{\{S_N>0\}}$ --- is instructive. 
Standard TCAV is the
Bayes-optimal estimator under $0$--$1$ loss
\begin{equation}
\ell(\widehat{\theta},\theta)
=
\mathds{1}_{\{\widehat{\theta}\neq\theta\}},
\end{equation}
which discards the magnitude of the posterior probability and is
therefore systematically overconfident when $|S_N|$ is small relative to
$\sigma/\sqrt{N}$: in this regime,
$\Phi(\sqrt{N}S_N/\sigma)$ is close to $1/2$, expressing
uncertainty, whereas standard TCAV outputs a definite $0$ or $1$.
Multi-TCAV, on the other hand, is underconfident relative to the calibrated
full-sample probability: each batch uses only $n=N/s$ samples, so its
effective signal-to-noise ratio is $\sqrt{n}/\sigma$ rather than
$\sqrt{N}/\sigma$. Recalling \eqref{eq:alpha_star},
\begin{equation}
\frac{\alpha^\star}{\alpha^\dagger}
=
\frac{1}{\sqrt{s-1}}
\leq1,
\end{equation}
with equality for $s=2$ and strict inequality for $s>2$. Thus, for
$s>2$, $\alpha^\star$ produces a flatter Probit-Normal score
that reproduces the more conservative mean behavior of Multi-TCAV.
Note that both $\alpha^\star$ and $\alpha^\dagger$ from
\eqref{eq:alpha_star} and \eqref{eq:alpha_dagger} depend on $\sigma$, which
is unknown in practice and must be estimated from the data. We comment on
this below.


\subsection{Per-Input Calibration for Class-Level $\alpha$-TCAV}
\label{app:per_input_alpha}

We use the notation from the main text, in particular $\vz_i$, $\tau_i$, and $t_i$ from
\eqref{eq:inputwise_directional_variance_main} and \eqref{eq:t_i}. 
The evaluation set is fixed and independent of the data used to estimate the CAV. 
Let
\[
S_i:=S_{C,k,l}(\vx_i), \qquad
\overline{\vw}_{\CAV}:=\E[\vw_\CAV], \qquad
\mu_i:=\vz_i^\top\overline{\vw}_{\CAV}.
\]
Under the Gaussian model, conditionally on the fixed input $\vx_i$,
\begin{equation}
S_i\sim\mathcal{N}(\mu_i,\tau_i^2).
\label{eq:inputwise_conditional_gaussian_app}
\end{equation}
Both $\mu_i$ and $\tau_i$ may vary across inputs, so the marginal
sensitivity distribution over the class is generally a Gaussian mixture.
Its dispersion combines variation in the mean sensitivities $\mu_i$ with
CAV-estimation uncertainty and therefore cannot replace the local noise
scale $\tau_i$. This motivates calibrating each input separately. The
covariance underlying $\tau_i$ is characterized in
Sec.~\ref{app:sec_cav_distributions}; in practice, its estimate is inserted
into \eqref{eq:inputwise_directional_variance_main}. All
expectations and variances below are taken with respect to the randomness
of the CAV training data.
For $\tau_i>0$, treat $\mu_i$ as unknown and $\tau_i$ as known, and adopt
the flat prior. Gaussian conjugacy gives
\begin{equation}
\mu_i\mid S_i
\sim
\mathcal{N}(S_i,\tau_i^2),
\qquad
\P(\mu_i>0\mid S_i)
=
\Phi\left(\frac{S_i}{\tau_i}\right)
=
\Phi(t_i).
\label{eq:inputwise_posterior_probability_app}
\end{equation}
Hence, $\kappa=1$ in \eqref{eq:inputwise_alpha_TCAV_class_main} gives the
posterior sign probability separately for every input. More generally,
write $p_i^{(\kappa)}:=\Phi(\kappa t_i)$, whose sharpness on the original
sensitivity scale is $\alpha_i=\kappa/\tau_i$.
Let $T_{\mathrm{multi},i}$ denote the inputwise Multi-TCAV score obtained
by splitting the total sampling budget into $s$ equal batches. Then, \eqref{eq:probit_normal_mean} gives 
\begin{equation}
\E \left[p_i^{(\kappa)}\right]
=
\Phi\left(
\frac{\kappa\mu_i}
{\tau_i\sqrt{1+\kappa^2}}
\right),
\qquad
\E[T_{\mathrm{multi},i}]
=
\Phi\left(
\frac{\mu_i}{\sqrt{s}\,\tau_i}
\right),
\label{eq:inputwise_mean_matching_app}
\end{equation}
where the second identity follows because splitting the budget multiplies
the directional variance by $s$. Matching these means for every $\mu_i$
yields $\kappa/\sqrt{1+\kappa^2}=1/\sqrt{s}$ and therefore
\begin{equation}
\begin{aligned}
\kappa^\dagger=1,
&\qquad
\alpha_i^\dagger=\frac{1}{\tau_i}
&&\text{(Bayesian calibration)},\\
\kappa^\star=\frac{1}{\sqrt{s-1}},
&\qquad
\alpha_i^\star=\frac{1}{\sqrt{s-1}\,\tau_i}
&&\text{(Multi-TCAV mean matching)}.
\end{aligned}
\label{eq:kappa_distinguished_choices_app}
\end{equation}
Moreover, $\Var(\kappa S_i/\tau_i)=\kappa^2$, so $\kappa^\dagger=1$ places every input at its own Probit--Normal phase transition.

\begin{remark}[Scale, sample size, and CAV rescaling]
\label{remark:kappa_scale_invariance_app}
The parameter $\kappa$ is dimensionless, while $1/\tau_i$ converts it to
the original sensitivity scale. If
$\Cov(\vw_\CAV(N))=\boldsymbol{\Omega}/N$, then
\begin{equation}
\tau_i^2=\frac{\sigma_i^2}{N},
\qquad
\sigma_i^2:=\vz_i^\top\boldsymbol{\Omega}\vz_i,
\qquad
\alpha_i=\frac{\kappa\sqrt{N}}{\sigma_i}.
\label{eq:kappa_sample_size_app}
\end{equation}
Thus, $\tau_i$ absorbs the dependence on the sample size, input, and layer, while the distinguished values of $\kappa$ remain unchanged. The standardization is also invariant to positive CAV rescaling. If $\widetilde{\vw}_{\CAV}=c\vw_\CAV$ for $c>0$, then $\widetilde S_i=cS_i$ and $\widetilde\tau_i=c\tau_i$, and hence $\widetilde S_i/\widetilde\tau_i=t_i$. A negative rescaling instead reverses the concept orientation. In particular, FastCAV and PatternCAV give identical standardized scores when they are positive scalar multiples constructed from the same samples.
\end{remark}


 \subsection{MSE Analysis $\alpha$-TCAV}
\label{app:alpha_tcav_bias_variance}

Again, we use the notation from \eqref{eq:alpha_tcav_moment_parameters} and \eqref{eq:alpha_tcav_bias_variance_mse} in the main paper and write $p_N:=\Phi(\sqrt{N}\mu/\sigma)$, where we assume $\mu\neq0$ unless stated otherwise. In the two extreme cases, we obtain the following
expressions for the bias-variance tradeoff in the MSE,
\begin{align}
\alpha\downarrow0:\quad
\bigl(
\operatorname{Bias}_\alpha(\mu),
V_\alpha(\mu),
\operatorname{MSE}_\alpha(\mu)
\bigr)
&\longrightarrow
\left(
\frac12-\theta(\mu),\,0,\,\frac14
\right), \nonumber\\
\alpha\to\infty:\quad
\bigl(
\operatorname{Bias}_\alpha(\mu),
V_\alpha(\mu),
\operatorname{MSE}_\alpha(\mu)
\bigr)
&\longrightarrow
\left(
p_N-\theta(\mu),\,
p_N(1-p_N),\,
\Phi\left(-\frac{\sqrt{N}|\mu|}{\sigma}\right)
\right).
\label{eq:alpha_tcav_bias_variance_limits}
\end{align}
Thus, $\alpha\downarrow0$ yields a deterministic score of $1/2$, whereas $\alpha\to\infty$ yields standard TCAV, $\mathds{1}_{\{S_N>0\}}$. For fixed $\mu\neq0$, the absolute bias decreases with $\alpha$ from $1/2$ to
\begin{equation}
\left|\operatorname{Bias}_\infty(\mu)\right|
=
\Phi\left(-\frac{\sqrt{N}|\mu|}{\sigma}\right).
\label{eq:standard_tcav_absolute_bias}
\end{equation}
The variance behaves differently. At the critical point $\mu=0$, the variance increases from $0$ to $1/4$. For stronger signals, however, the variance may attain its maximum at an intermediate $\alpha$ and then decrease as the score becomes nearly deterministic. The MSE may consequently decrease, increase, or possess an interior minimum, depending on $\mu$, $N$, and $\sigma$. At $\mu=0$, the binary target $\theta(\mu)$ is discontinuous, and the value $1/2$ may provide a more natural representation of neutrality.
For a fixed sampling budget $N$, splitting the data into $s$ batches gives
\begin{equation}
\E[T_{\mathrm{multi}}]
=
\Phi\left(\frac{\sqrt{N}\mu}{\sigma\sqrt{s}}\right).
\label{eq:multi_tcav_mean_bias_appendix}
\end{equation}
Consequently,
\begin{equation}
\left|\E[T_{\mathrm{multi}}]-\theta(\mu)\right|
=
\Phi\left(-\frac{\sqrt{N}|\mu|}{\sigma\sqrt{s}}\right)
>
\Phi\left(-\frac{\sqrt{N}|\mu|}{\sigma}\right)
=
\left|\operatorname{Bias}_\infty(\mu)\right|
\label{eq:multi_tcav_bias_comparison}
\end{equation}
for every $\mu\neq0$ and $s>1$. Multi-TCAV is therefore more biased toward $1/2$ than standard TCAV under the pointwise binary target. The choice $\alpha^\star$ matches this mean and thus has exactly the same pointwise bias, while its variance is no larger.
This comparison is pointwise in the input. When scores are averaged across inputs, positive and negative biases may cancel, and the shared CAV induces the cross-input dependence. The pointwise inequality therefore does not by itself imply an analogous ordering for every class-level bias or variance.

\begin{figure*}[t]
    \centering
    \begin{tikzpicture}
        \begin{groupplot}[
            group style={
                group size=2 by 1,
                horizontal sep=1cm
            },
            width=0.45\textwidth,
            height=0.31\textwidth,
            xmode=log,
            xmin=0.02,
            xmax=50,
            ymin=0,
            ymax=0.32,
            xtick={0.03,0.1,0.3,1,3,10,30},
            xticklabels={$0.03$,$0.1$,$0.3$,$1$,$3$,$10$,$30$},
            ytick={0,0.05,0.10,0.15,0.20,0.25,0.30},
            xlabel={Normalized sharpness
                $\kappa=\alpha\sigma/\sqrt{N}$},
            ylabel={Error component},
            tick label style={font=\scriptsize},
            label style={font=\small},
            title style={font=\small},
            axis line style={black!75},
            legend style={
                at={(0.5,-0.30)},
                anchor=north,
                legend columns=2,
                draw=none,
                fill=none,
                font=\scriptsize,
                row sep=-1pt,
                /tikz/every even column/.append style={
                    column sep=5pt
                }
            },
            legend cell align={left},
            every axis plot/.append style={line width=1.15pt}
        ]

        \nextgroupplot[
            title={(a) Bias--variance decomposition, $N=100$}
        ]
            \addplot[biasblue]
                table[x=k,y=biasSqN100] {\alphatcavdata};
            \addlegendentry{
                $\operatorname{Bias}_\alpha(\mu)^2$
            }

            \addplot[varianceorange, dash dot]
                table[x=k,y=varN100] {\alphatcavdata};
            \addlegendentry{
                $\operatorname{Var}_\mu(Y_\alpha)$
            }

            \addplot[black]
                table[x=k,y=mseN100] {\alphatcavdata};
            \addlegendentry{
                $\operatorname{MSE}_\alpha(\mu)$
            }

            \addplot[gray!70, densely dotted]
                coordinates {(0.5,0) (0.5,0.32)};
            \addlegendentry{
                $\kappa^\star=1/2$ ($\alpha^\star$, $s=5$)
            }

            \addplot[gray!70, densely dashed]
                coordinates {(1,0) (1,0.32)};
            \addlegendentry{
                $\kappa^\dagger=1$ ($\alpha^\dagger$)
            }

        \nextgroupplot[
            title={(b) MSE for different sample sizes},
            ylabel={}
        ]
            \addplot[varianceorange, dashed]
                table[x=k,y=mseN25] {\alphatcavdata};
            \addlegendentry{$N=25$}

            \addplot[biasblue]
                table[x=k,y=mseN100] {\alphatcavdata};
            \addlegendentry{$N=100$}

            \addplot[msegreen, dash dot]
                table[x=k,y=mseN400] {\alphatcavdata};
            \addlegendentry{$N=400$}

            \addplot[gray!70, densely dashed]
                coordinates {(1,0) (1,0.32)};
            \addlegendentry{
                $\kappa^\dagger=1$ ($\alpha^\dagger$)
            }

            \addplot[
                only marks,
                mark=*,
                mark size=1.6pt,
                varianceorange,
                forget plot
            ] coordinates {(0.60753,0.19900)};

            \addplot[
                only marks,
                mark=*,
                mark size=1.6pt,
                biasblue,
                forget plot
            ] coordinates {(1.18403,0.11244)};

            \addplot[
                only marks,
                mark=*,
                mark size=1.6pt,
                msegreen,
                forget plot
            ] coordinates {(2.32975,0.01716)};
        \end{groupplot}
    \end{tikzpicture}
    \caption{Bias--variance behavior of $\alpha$-TCAV under the
    fixed-input Gaussian model with $\mu=0.1$ and $\sigma=1$. The
    horizontal axis uses the normalized sharpness
    $\kappa=\alpha\sigma/\sqrt{N}$. Panel~(a) displays the exact
    decomposition
    $\operatorname{MSE}_\alpha(\mu)
    =\operatorname{Bias}_\alpha(\mu)^2
    +\operatorname{Var}_\mu(Y_\alpha)$ at $N=100$. The gray dotted
    line marks $\kappa^\star=1/\sqrt{s-1}=1/2$ for $s=5$, while the
    gray dashed line marks the Bayes-calibrated value
    $\kappa^\dagger=1$. Panel~(b) shows that the pointwise MSE and its
    minimizing sharpness depend on $N$; the dots indicate the numerical
    minima. As $\kappa\downarrow0$, the score converges to $1/2$ and
    the MSE to $1/4$, whereas $\kappa\to\infty$ recovers standard TCAV.}
    \label{fig:alpha_tcav_bias_variance}
\end{figure*}

\subsection{Class-level variance and error guarantees}
\label{app:class_level_rb}

For PatternCAV, the model \eqref{eq:rb_main_joint_model} holds under
independent Gaussian activations when each group is partitioned into
$s$ equally sized batches, with covariance given by
\eqref{eq:w_pat_covariance}. The result below uses the true covariance;
substituting an estimate need not preserve the exact identity.

\begin{theorem}[Class-level variance and error reduction]
\label{thm:class_level_rb}
Under \eqref{eq:rb_main_joint_model} and the accompanying assumptions,
\eqref{eq:rb_main_identity} holds. In particular, the two scores have
the same expectation, and, for every fixed $\theta\in\mathbb R$,
\begin{align}
 \Var(T_{\mathrm{multi}})
 -\Var\!\left(\TCAV_{C,k,l}^{(\kappa^\star)}\right)
 &=\E\!\left[
 \left(T_{\mathrm{multi}}-\TCAV_{C,k,l}^{(\kappa^\star)}\right)^2
 \right]\geq0,
 \label{eq:rb_variance_gain}\\
 \E[(T_{\mathrm{multi}}-\theta)^2]
 -\E\!\left[
 \left(\TCAV_{C,k,l}^{(\kappa^\star)}-\theta\right)^2
 \right]
 &=\E\!\left[
 \left(T_{\mathrm{multi}}-\TCAV_{C,k,l}^{(\kappa^\star)}\right)^2
 \right].
 \label{eq:rb_mse_gain}
\end{align}
The improvement is strict precisely when
$\E[\Var(T_{\mathrm{multi}}\mid\vw_\CAV)]>0$.
The target $\theta$ may, in particular, be the class score in
\eqref{eq:TCAV_def_original_indicator_fcts} evaluated at the
population mean CAV $\E[\vw_\CAV]$.
\end{theorem}

\begin{proof}
Since
$\Cov(\vw_\CAV^{(j)},\vw_\CAV)=\Cov(\vw_\CAV)=\mSigma$,
the Gaussian residual $\vw_\CAV^{(j)}-\vw_\CAV$ is independent of
$\vw_\CAV$ and has covariance $(s-1)\mSigma$. Hence
\begin{equation}
 \vw_\CAV^{(j)}\mid\vw_\CAV
 \sim\mathcal N_{d_l}\!\left(\vw_\CAV,(s-1)\mSigma\right).
 \label{eq:rb_conditional_batch}
\end{equation}
Projecting onto $\vz_i$ and using \eqref{eq:t_i}, the conditional
probability of a positive batch sensitivity is
$\Phi(\kappa^\star t_i)$. Therefore, for the inputwise scores from
Appendix~\ref{app:per_input_alpha},
\begin{equation}
 \E[T_{\mathrm{multi},i}\mid\vw_\CAV]
 =p_i^{(\kappa^\star)}.
 \label{eq:rb_conditional_sign}
\end{equation}
Averaging over inputs yields \eqref{eq:rb_main_identity}, without
requiring independence between their scores. The law of total
variance gives
\[
 \Var(T_{\mathrm{multi}})
 =\Var\!\left(\TCAV_{C,k,l}^{(\kappa^\star)}\right)
  +\E[\Var(T_{\mathrm{multi}}\mid\vw_\CAV)].
\]
The last term equals the expected squared deviation from the
conditional mean, proving \eqref{eq:rb_variance_gain} and the
strictness criterion. Equal expectations imply equal biases
relative to $\theta$, so the bias--variance decomposition gives
\eqref{eq:rb_mse_gain}.
\end{proof}

The same argument applies to every fixed weighted aggregation of
the inputwise scores in \eqref{eq:rb_conditional_sign}.
Thus the variance and mean-squared-error guarantees also extend
to subgroup averages and contrasts.

\section{Characterizing CAV Distributions.}
\label{app:sec_cav_distributions}

We characterize the first two moments of ${\vw}_{\fast}$ from  \eqref{eq:fastCAV_definition} ${\vw}_{\pat}$ from \eqref{eq:w_pat_empirical} in terms of the corresponding first two moments of the underlying data distribution, generalizing \citep[Prop. 1]{schnoor2026cav}.

\begin{proposition}[First two moments of ${\vw}_{\fast}$ and ${\vw}_{\pat}$]
\label{prop:distribution_fast_pattern_CAV}
\label{prop:hyperplane}
Mean and covariance of ${\vw}_{\fast}$ and ${\vw}_{\pat}$ can be expressed in terms of the data distribution 
(class-specific mean $\vmu_\ell$ and covariance $\mSigma_\ell$, $\ell=1,2$) as
\begin{align}
    \E[\vw_{\pat}] & =  \vmu_2 - \vmu_1 \in \R^d, \label{eq:w_pat_mean} \\
 \Cov(\vw_{\pat})  & =  \frac{1}{n_1} \mSigma_1 + \frac{1}{n_2} \mSigma_2 \in \R^{d \times d}. \label{eq:w_pat_covariance}
\end{align}
The \textit{PatternCAV} $\vw_{\pat}$ is a scaled version of the \textit{FastCAV} $\vw _{\fast}$; 
for general $n_1, n_2 \in \N$ it holds that
\begin{equation}
   \vw _{\fast} = \frac{n_1}{n_1 + n_2} \vw_{\pat}.
   \label{eq:scaling_pattern_fast_cav_general}
\end{equation}
Therefore, mean and covariance of $\vw_{\fast}$  are given as appropriately scaled versions 
of \eqref{eq:w_pat_mean} and \eqref{eq:w_pat_covariance},
\begin{align}
   \E [\vw_{\fast} ]  & =   \frac{n_1}{n_1 + n_2} (\vmu_2 - \vmu_1) \in \R^d, \label{eq:fastCAV_mean_general}    \\ 
\Cov( \vw_{\fast} )   & =   \left( \frac{n_1}{n_1 + n_2} \right)^2 \left( \frac{1}{n_1} \mSigma_1 + \frac{1}{n_2} \mSigma_2 \right).  \label{eq:fastCAV_cov_general}
\end{align}
\end{proposition}

The next corollary follows immediately as a special case by assuming classes of equal size.

\begin{corollary}[Balanced case]
\label{cor:distribution_fast_pattern_CAV_balanced_case}
In the case of $n_1 = n_2$, the relation \eqref{eq:scaling_pattern_fast_cav_general} simplifies to $ \vw_{\fast} = \frac{1}{2} \vw_{\pat}$, when \eqref{eq:fastCAV_mean_general} and \eqref{eq:fastCAV_cov_general} now become
(note that $n_1, n_2$ in the covariance term can be used interchangeably)
\begin{equation}
   \E [ \vw_{\fast} ]
= \frac{\vmu_2 - \vmu_1}{2} \in \R^d,  \qquad 
     \Cov( \vw_{\fast} )
 =   \frac{1}{4 n_1}\mSigma_1 + \frac{1}{4 n_2}\mSigma_2 \in \R^{d \times d}.  
 \label{eq:fastCAV_mean_and_cov_definition_balanced}
\end{equation}
\end{corollary}

Further, \citep[Remark 1]{schnoor2026cav} pointed out a close relation of $\vw_{\fast}$ (and, in the balanced case as in Corollary 
\eqref{cor:distribution_fast_pattern_CAV_balanced_case}, also $\vw_{\pat}$) to the case of CAVs based on ridge regression.
Finally, let us also consider the case of a fixed number of \textit{concept} samples $n_2$, while taking the limit $n_1 \to \infty$
for the \textit{non-concept} data, which is the scenario studied in \citep{wenkmann2025variability}. Interestingly, in this setup
limiting mean and covariance $\vw_{\pat}$ and $\vw_{\fast}$ coincide, showing a similar behavior to the balanced case $n_1 = n_2$.

\begin{corollary}[Asymmetrically asymptotic: $n_1 \to \infty$]
\label{cor:distribution_fast_pattern_CAV_asymetric_asymptotic}
For a fixed number $n_2$ of concept samples, and $n_1 \to \infty$ for the number of non-concept samples,  
we obtain for $\vw_{\pat}$ and $\vw_{\fast} $ that
\begin{alignat}{2}
 \E[\vw_{\pat}] & =  \vmu_2 - \vmu_1 \in \R^d, &\qquad  \Cov(\vw_{\pat})  & \to  \frac{1}{n_2} \mSigma_2 \in \R^{d \times d},    \\
 \E [\vw_{\fast} ]  & \to \vmu_2 - \vmu_1 \in \R^d,    &\qquad  \Cov( \vw_{\fast} )   & \to  \frac{1}{n_2} \mSigma_2 \in \R^{d \times d}. \label{eq:fast_pattern_CAV_cov_general_asymetric_asymptotic}
\end{alignat}

Here, by $\to$ we denote convergence w.r.t. to any vector or matrix norm on $\R^d$ or $\R^{d \times d}$, respectively.
\end{corollary}

Informally, as $n_1 \to \infty$, randomness w.r.t. to the \textit{non-concept} data vanishes, and the covariance matrices in  \eqref{eq:fast_pattern_CAV_cov_general_asymetric_asymptotic} depend solely on the variability of the \textit{concept} data.

\end{document}